\documentclass{article}

\usepackage[final]{neurips_2020}

\usepackage[utf8]{inputenc} 
\usepackage[T1]{fontenc}    
\usepackage{hyperref}       
\usepackage{url}            
\usepackage{booktabs}       
\usepackage{amsfonts}       
\usepackage{nicefrac}       
\usepackage{microtype}      
\usepackage{graphicx,color}
\usepackage{latexsym,amsmath,amssymb,amsfonts}
\usepackage{comment}
\usepackage{algorithm}
\usepackage[noend]{algpseudocode}
\usepackage{amsthm}
\usepackage{times}
\usepackage{subcaption}
\usepackage{mathabx}
\usepackage{multirow}
\usepackage{array}
\usepackage{wrapfig}
\usepackage{booktabs}
\newcolumntype{P}[1]{>{\centering\arraybackslash}p{#1}}
\newcolumntype{M}[1]{>{\centering\arraybackslash}m{#1}}

\hypersetup{colorlinks,
            linkcolor=blue,
            citecolor=blue,
            urlcolor=magenta,
            linktocpage,
            plainpages=false}

\newtheorem{thm}{Theorem}[section]
\newtheorem{lem}{Lemma}[section]

\newtheorem{cond}{Condition}[section]
\newtheorem*{remark}{Remark}

\title{Efficient Variational Inference for Sparse Deep Learning with Theoretical Guarantee}

\author{%
  Jincheng Bai \\
  Department of Statistics\\
  Purdue University\\
  West Lafayette, IN 47906 \\
  \texttt{bai45@purdue.edu} \\
   \And
  Qifan Song \\
  Department of Statistics\\
  Purdue University\\
  West Lafayette, IN 47906 \\
  \texttt{qfsong@purdue.edu} \\
   \AND
  Guang Cheng \\
  Department of Statistics\\
  Purdue University\\
  West Lafayette, IN 47906 \\
  \texttt{chengg@purdue.edu} \\
}

\begin{document}

\maketitle

\begin{abstract}
Sparse deep learning aims to address the challenge of huge storage consumption by deep neural networks, and to recover the sparse structure of target functions. Although tremendous empirical successes have been achieved, most sparse deep learning algorithms are lacking of theoretical support. On the other hand, another line of works have proposed theoretical frameworks that are computationally infeasible. In this paper, we train sparse deep neural networks with a fully Bayesian treatment under spike-and-slab priors, and develop a set of computationally efficient variational inferences via continuous relaxation of Bernoulli distribution. The variational posterior contraction rate is provided, which justifies the consistency of the proposed variational Bayes method. Notably, our empirical results demonstrate that this variational procedure provides uncertainty quantification in terms of Bayesian predictive distribution and is also capable to accomplish consistent variable selection by training a sparse multi-layer neural network.
\end{abstract}

\section{Introduction}
Dense neural network (DNN) may face various problems despite its huge successes in AI fields. Larger training sets and more complicated network structures improve accuracy in deep learning, but always incur huge storage and computation burdens. For example, small portable devices may have limited resources such as several megabyte memory, while a dense neural networks like ResNet-50 with 50 convolutional layers would need more than 95 megabytes of memory for storage and numerous floating number computation \citep{Cheng2018Model}. It is therefore necessary to compress deep learning models before deploying them on these hardware limited devices. 

In addition, sparse neural networks may recover the potential sparsity structure of the target function, e.g., sparse teacher network in the teacher-student framework \citep{Goldt2019Dynamics, tian2018theoretical}. Another example is from nonparametric regression with sparse target functions, i.e., only a portion of input variables are relevant to the response variable. A sparse network may serve the goal of variable selection \citep{Feng2017Sparse, Liang2018Bayesian, Ye2018variable}, and is also known to be robust to adversarial samples against $l_{\infty}$ and $l_2$ attacks \citep{Guo2018Sparse}. 

Bayesian neural network (BNN), which dates back to \cite{Mackay1992practical,Neal1993Bayesian}, comparing with frequentist DNN,  possesses the advantages of robust prediction via model averaging and automatic uncertainty quantification \citep{Blundell2015weight}. Conceptually, BNN can easily induce sparse network selection by assigning discrete prior over all possible network structures. However, challenges remain for sparse BNN  including inefficient Bayesian computing issue and the lack of theoretical justification. This work aims to resolve these two important bottlenecks simultaneously, by utilizing variational inference approach \citep{Jordan1999Variational, Blei2017variational}. On the computational side, it can reduce the ultra-high dimensional sampling problem of Bayesian computing, to an optimization task which can still be solved by a back-propagation algorithm. On the theoretical side, we provide a proper prior specification, under which the variational posterior distribution converges towards the truth. To the best of our knowledge, this work is the first one that provides a complete package of both theory and computation for sparse Bayesian DNN. 


\paragraph{Related work.}
A plethora of methods on sparsifying or compressing neural networks have been proposed \citep{Cheng2018Model, Gale2019state}. The majority of these methods are pruning-based \citep{Han2016Deep, Zhu2018Toprune, Frankle2018the}, which are ad-hoc on choosing the threshold of pruning and usually require additional training and fine tuning. Some other methods could achieve sparsity during training. For example, \cite{Louizos2018learning} introduced $l_0$ regularized learning and \cite{Mocanu2018scalable} proposed sparse evolutionary training. However, the theoretical guarantee and the optimal choice of hyperparameters for these methods are unclear.
As a more natural solution to enforce sparsity in DNN, Bayesian sparse neural network has been proposed by placing prior distributions on network weights: \cite{Blundell2015weight} and \cite{Deng2019Adaptive} considered spike-and-slab priors with a Gaussian and Laplacian spike respectively; Log-uniform prior was used in \cite{Molchanov2017Variational}; \cite{GHosh2018Structured} chose to use the popular horseshoe shrinkage prior. These existing works actually yield posteriors over the dense DNN model space despite applying sparsity induced priors. In order to derive
explicit sparse inference results, users have to additionally determine certain pruning rules on the posterior. On the other hand, theoretical works regarding sparse deep learning have been studied in \cite{Schmidt-Hieber2017Nonparametric}, \cite{Polson2018posterior} and \cite{Cherief2020convergence}, but finding an efficient implementation to close the gap between theory and practice remains a challenge for these mentioned methods.

\paragraph{Detailed contributions.}
We impose a spike-and-slab prior on all the edges (weights and biases) of a neural network, where the spike component and slab component represent whether the corresponding edge is inactive or active, respectively. Our work distinguished
itself from prior works on Bayesian sparse neural network by imposing the spike-and-slab prior with the Dirac spike
function. Hence automatically, all posterior samples are from exact sparse DNN models.

More importantly, with carefully chosen hyperparameter values, especially the prior probability that each edge is active, we establish the variational posterior consistency, and the corresponding convergence rate strikes the balance of statistical estimation error, variational error and the approximation error.
The theoretical results are validated by various simulations and real applications. Empirically we also demonstrate that the proposed method possesses good performance of variable selection and  uncertainty quantification.
While \cite{Feng2017Sparse, Liang2018Bayesian, Ye2018variable} only considered the neural network with single hidden layer for variable selection, we observe correct support recovery for neural networks with multi-layer networks.

\section{Preliminaries}

\paragraph{Nonparametric regression.}{\label{secNonparam}}
Consider a nonparametric regression model with random covariates
\begin{equation}{\label{eqmod1}}
   Y_i = f_0(X_i)+\epsilon_i, \: i=1,\ldots,n, 
\end{equation}
where $X_i= (x_{i1}, \ldots, x_{ip})^t \sim \mathcal{U}([-1,1]^p)$\footnote{This compact support assumption is generally satisfied given the standardized data, and may be relaxed.}, 
$\mathcal{U}$ denotes the uniform distribution, $\epsilon_i\overset{iid}{\sim}\mathcal{N}(0,\sigma^2_{\epsilon})$ is the noise term, and $f_0:[-1,1]^p\rightarrow \mathbb{R}$ is an underlying true function. For simplicity of analysis, we assume $\sigma_{\epsilon}$ is known. Denote $D_i=(X_i, Y_i)$ and $D=(D_1,\dots, D_n)$ as the observations. Let $P_0$ denote the underlying probability measure of data, and $p_0$ denote the corresponding density function. 


\paragraph{Deep neural network.}

An $L$-hidden-layer neural network is used to model the data. 
The number of neurons in each hidden layer is denoted by $p_i$ for $i=1,\dots, L$. The weight matrix and bias parameter in each layer are denoted by $W_i\in\mathbb{R}^{p_{i-1}\times p_{i}}$ and $b_i\in\mathbb{R}^{p_{i}}$ for $i=1,\dots,L+1$.
Let $\sigma(x)$ be the activation function, and for any $r\in\mathbb{Z}^+$ and any $b\in\mathbb{R}^r$, we define $\sigma_b: \mathbb{R}^r \rightarrow \mathbb{R}^r$ as $\sigma_b(y_j) = \sigma(y_j-b_j)$ for $j=1,\ldots,r$.
Then, given parameters $\boldsymbol{p}=(p_1,\dots,p_L)'$ and $\theta = \{W_1, b_1,\ldots, W_{L}, b_{L}, W_{L+1}, b_{L+1}\}$,
the output of this DNN model can be written as 
\begin{equation}{\label{eqmod2}}
    f_{\theta}(X)=W_{L+1}\sigma_{b_{L}}(W_{L}\sigma_{b_{L-1}}\ldots \sigma_{b_1}(W_1X)) + b_{L+1}.
\end{equation}
In what follows, with slight abuse of notation, $\theta$ is also viewed as a vector that contains all the coefficients in $W_i$'s and $b_i$'s, , i.e., $\theta=(\theta_1,\dots,\theta_T)'$, where the length $T:=\sum^{L-1}_{l=1}p_{l+1}(p_l+1) + p_1(p+1) + (p_L+1)$.

\paragraph{Variational inference.}
Bayesian procedure makes statistical inferences from the posterior distribution $\pi(\theta|D)\propto \pi(\theta)p_{\theta}(D)$, where $\pi(\theta)$ is the prior distribution. Since MCMC doesn't scale well on complex Bayesian learning tasks with large datasets, variational inference \citep{Jordan1999Variational, Blei2017variational} has become a popular alternative. Given a variational family of distributions, denoted by $\mathcal Q$ \footnote{For simplicity, it is commonly assumed that $\mathcal{Q}$ is the mean-field family, i.e. $
q(\theta) = \prod^T_{i=1}q(\theta_i)$.}, it seeks to approximate the true posterior distribution by finding a closest member of $\mathcal Q$ in terms of KL divergence: 
\begin{equation}\label{eqvb1}
\widehat{q}(\theta) = \underset{q(\theta) \in \mathcal{Q}}{\mbox{argmin } } \mbox{KL}(q(\theta)||\pi(\theta|D)).
\end{equation}
The optimization ($\ref{eqvb1}$) is equivalent to minimize the negative ELBO, which is defined as
\begin{equation}\label{eqvb3}
\begin{split}
\Omega = -\mathbb{E}_{q(\theta)}[\log p_{\theta}(D)] + \mbox{KL}(q(\theta)||\pi(\theta)),
\end{split}
\end{equation}
where the first term in ($\ref{eqvb3}$) can be viewed as the reconstruction error \citep{Kingma2014VAE} and the second term serves as regularization. Hence, the variational inference procedure minimizes the reconstruction error while being penalized against prior distribution in the sense of KL divergence.

When the variational family is indexed by some hyperparameter $\omega$, i.e., any $q\in \mathcal Q$ can be written as $q_\omega(\theta)$, then the negative ELBO is a function of $\omega$ as $\Omega(\omega)$. 
The KL divergence term in ($\ref{eqvb3}$) could usually be integrated analytically, while the reconstruction error requires Monte Carlo estimation. Therefore, the optimization of $\Omega(\omega)$ can utilize the stochastic gradient approach \citep{Kingma2014VAE}. To be concrete, if all distributions in $\mathcal Q$ can be reparameterized as $q_\omega\stackrel{d}{=}g(\omega,\nu)$\footnote{``$\stackrel{d}{=}$'' means equivalence in distribution} for some differentiable function $g$ and random variable $\nu$, then
the stochastic estimator of $\Omega(\omega)$ and its gradient are
\begin{equation} \label{eq:sgvb}
\begin{split}
&\widetilde{\Omega}^m(\omega) = -\frac{n}{m}\frac{1}{K}\sum^m_{i=1}\sum^K_{k=1}\log p_{g(\omega, \nu_k)}(D_i) + \mbox{KL}(q_\omega(\theta)||\pi(\theta))\\
&\nabla_{\omega}\widetilde{\Omega}^m(\omega) = -\frac{n}{m}\frac{1}{K}\sum^m_{i=1}\sum^K_{k=1}\nabla_{\omega}\log p_{g(\omega, \nu_k)}(D_i) + \nabla_{\omega}\mbox{KL}(q_\omega(\theta)||\pi(\theta)),
\end{split}
\end{equation}
where $D_i$'s are randomly sampled data points and $\nu_k$'s are iid copies of $\nu$. Here, $m$ and $K$ are minibatch size and Monte Carlo sample size, respectively. 

\section{Sparse Bayesian deep learning with spike-and-slab prior}
We aim to approximate $f_0$ in the generative model (\ref{eqmod1}) by a sparse neural network. Specifically, given a network structure, i.e. the depth $L$ and the width $\boldsymbol{p}$, $f_0$ is approximated by DNN models $f_\theta$ with sparse parameter vector $\theta\in\Theta=\mathbb R^T$. From a Bayesian perspective, we impose a spike-and-slab prior \citep{George1993variable, Ishwaran2005Spike} on $\theta$ to model sparse DNN. 

A spike-and-slab distribution is a mixture of two components: a Dirac spike concentrated at zero and a flat slab distribution. Denote $\delta_{0}$ as the Dirac at 0 and $\gamma = (\gamma_1, \ldots, \gamma_T)$ as a binary vector indicating the inclusion of each edge in the network. The prior distribution $\pi(\theta)$ thus follows:
\begin{equation}\label{prior}
\begin{split}
\theta_i|\gamma_i \sim \gamma_i\mathcal{N}(0, \sigma^2_0) + (1-
\gamma_i)\delta_0, \quad \gamma_i \sim \mbox{Bern}(\lambda), 
\end{split}
\end{equation}
for $i=1,\ldots, T$, 
where $\lambda$ and $\sigma^2_0$ are hyperparameters representing the prior inclusion probability and the prior Gaussian variance, respectively. The choice of $\sigma_0^2$ and $\lambda$ play an important role in sparse Bayesian learning, and in Section \ref{sec:contract}, we will establish theoretical guarantees for the variational inference procedure under proper deterministic choices of $\sigma_0^2$ and $\lambda$. Alternatively, hyperparameters may be chosen via an Empirical Bayesian (EB) procedure, but it is beyond the scope of this work. We assume $\mathcal{Q}$ is in the same family of spike-and-slab laws:
\begin{equation}\label{vbpost}
\begin{split}
  \theta_i|\gamma_i \sim \gamma_i\mathcal{N}(\mu_i, \sigma^2_i)+(1-\gamma_i)\delta_{0}, \quad \gamma_i \sim \mbox{Bern}(\phi_i)
\end{split}
\end{equation}
for $i=1,\ldots, T$, where $0\leq \phi_i \leq 1$. 

Comparing to pruning approaches \citep[e.g.][]{Zhu2018Toprune, Frankle2018the,Molchanov2017Variational} that don't pursue sparsity among bias parameter $b_i$'s, the Bayesian modeling induces posterior sparsity for both weight and bias parameters. 

In the literature, \cite{Polson2018posterior,Cherief2020convergence} imposed sparsity specification as follows $\Theta(L,\boldsymbol{p},s)=\{\theta \mbox{ as in model } (\ref{eqmod2}): ||\theta||_{0}\leq s\}$ that not only posts great computational challenges, but also requires tuning for optimal sparsity level $s$. For example, \cite{Cherief2020convergence} shows that given $s$, two error terms occur in the variation DNN inference: 1) the variational error $r_n(L, \boldsymbol{p}, s)$ caused by the variational Bayes approximation to the true posterior distribution and 2) the approximation error $\xi_n(L, \boldsymbol{p}, s)$ between $f_0$ and the best bounded-weight $s$-sparsity DNN approximation of $f_0$. Both error terms $r_n$ and $\xi_n$ depend on $s$ (and their specific forms are given in next section). Generally speaking, as the model capacity (i.e., $s$) increases, $r_n$ will increase and $\xi_n$ will decrease. Hence the optimal choice $s^*$ that strikes the balance between these two is
\[s^*= \underset{s}{\mbox{argmin } } \{r_n(L,\boldsymbol{p},s) + \xi_n(L,\boldsymbol{p},s)\}.\]
Therefore, one needs to develop a selection criteria for $\widehat s$ such that $\widehat s\approx s^*$. In contrast, our modeling directly works on the whole sparsity regime without pre-specifying $s$, and is shown later to be capable of automatically attaining the same rate of convergence as if the optimal $s^*$ were known.

\section{Theoretical results} \label{sec:contract}
In this section, we will establish the contraction rate of the variational sparse DNN procedure, without knowing $s^*$. For simplicity, we only consider equal-width neural network (similar as \cite{Polson2018posterior}). 

The following assumptions are imposed:
\begin{cond} \label{cond:width}
 $p_i\equiv N\in\mathbb Z^{+}$ that can depend on n, and $\lim T=\infty$.
\end{cond}
\begin{cond} \label{cond:activate}
  $\sigma(x)$ is 1-Lipschitz continuous. 
\end{cond}

\begin{cond}{\label{lambda}}
The hyperparameter $\sigma_0^2$ is set to be some constant, and $\lambda$ satisfies
$\log(1/\lambda) = O\{(L+1)\log N + \log(p\sqrt{n/s^*})\}$ and 
$\log(1/(1-\lambda)) = O((s^*/T)\{(L+1)\log N + \log(p\sqrt{n/s^*})\})$.
\end{cond}
Condition \ref{cond:activate} is very mild, and includes ReLU, sigmoid and tanh. Note that Condition \ref{lambda} gives a wide range choice of $\lambda$, even including the choice of $\lambda$ independent of $s^*$ (See Theorem \ref{mainthm} below).


We first state a lemma on an upper bound for the negative ELBO.
Denote the log-likelihood ratio between $p_0$ and $p_{\theta}$ as 
$l_n(P_0, P_{\theta}) = \log(p_0(D)/p_{\theta}(D)) = \sum^n_{i=1} \log(p_0(D_i)/p_{\theta}(D_i))$. Given some constant $B>0$, we define 
\begin{eqnarray*}
r^*_n&:=&r_n(L, N, s^*)=((L+1)s^*/n)\log N + (s^*/n)\log(p\sqrt{n/s^*}),\\
\xi^*_n&:=&\xi_n(L, N, s^*)= \inf_{\theta \in \Theta(L, \boldsymbol{p}, s^*),\|\theta\|_\infty\leq B}||f_{\theta}-f_0||^2_\infty.
\end{eqnarray*}
Recall that $r_n(L, N, s)$ and $\xi_n(L, N, s)$ denote the variational error and the approximation error. 

\begin{lem}{\label{thmbound2}} 
Under Condition \ref{cond:width}-\ref{lambda}, then with dominating probability, 
\begin{equation}{\label{eqbound2}}
\begin{split}
\inf_{q(\theta) \in \mathcal Q}\Bigl\{ \mbox{KL}(q(\theta)||\pi(\theta|\lambda))
+ \int_{\Theta} l_n(P_0, P_{\theta})q(\theta)d\theta \Bigr\} \leq Cn (r^*_n
+ \xi^*_n )
\end{split}
\end{equation}
where $C$ is either some positive constant if $\lim n(r_n^*+\xi^*_n)=\infty$, or any diverging sequence if $\lim\sup n(r_n^*+\xi^*_n)\neq \infty$.
\end{lem}

Noting that $\mbox{KL}(q(\theta)||\pi(\theta|\lambda))
+ \int_{\Theta} l_n(P_0, P_{\theta})q(\theta)(d\theta)$ is the negative ELBO up to a constant, we therefore show the optimal loss function of the proposed variational inference is bounded.

The next lemma investigates the convergence  of the variational distribution under the Hellinger distance, which is defined as 
$$d^2(P_{\theta}, P_0) = \mathbb{E}_X\Bigl( 1 - \exp \{-[f_{\theta}(X) - f_0(X)]^2/(8\sigma^2_{\epsilon}) \}\Bigr).$$ In addition, let $s_n=s^*\log^{2\delta-1}(n)$ for any $\delta>1$. An assumption on $s^*$ is required to strike the balance between $r_n^*$ and $\xi^*$:
\begin{cond} \label{cond:sieve}
 $\max\{s^*\log(p\sqrt{n/s^*}, (L+1)s^*\log N\} = o(n)$ and $r_n^*\asymp \xi_n^*$.
\end{cond}

\begin{lem}{\label{lmbound1}}
Under Conditions \ref{cond:width}-\ref{cond:sieve}, if $\sigma_0^2$ is set to be constant and $\lambda\leq T^{-1}\exp\{-M nr_n^*/s_n\}$ for any positive diverging sequence $M\rightarrow\infty$, then with dominating probability, we have
\begin{equation}\label{eqbound1}
\begin{split}
\int_\Theta d^2(P_{\theta}, P_0)\widehat{q}(\theta)d\theta \leq C\varepsilon^{*2}_n + \frac{3}{n}\inf_{q(\theta) \in \mathcal Q}\Bigl\{ \mbox{KL}(q(\theta)||\pi(\theta|\lambda))
+ \int_{\Theta} l_n(P_0, P_{\theta})q(\theta)d\theta \Bigr\},
\end{split}
\end{equation}
where $C$ is some constant, and $$\varepsilon^*_n :=\varepsilon_n(L,N,s^*)=\sqrt{r_n(L,N,s^*)}\log^\delta(n), \mbox{ for any } \delta > 1.$$
\end{lem}

\begin{remark}
The result (\ref{eqbound1}) is of exactly the same form as in the existing literature \citep{Pati2018on}, but it is not trivial in the following sense. The existing literature require the existence of a global testing function that separates $P_0$ and $\{P_{\theta}:d(P_{\theta}, P_0) \geq \varepsilon_n^*\}$ with exponentially decay rate of Type I and Type II errors. If such a testing function exists only over a subset $\Theta_n\subset\Theta$ (which is the case for our DNN modeling), then the existing result \citep{Yang2020alpha} can only characterize the VB posterior contraction behavior within $\Theta_n$, but not over the whole parameter space $\Theta$. Therefore our result, which characterizes the convergence behavior for the overall VB posterior, represents a significant improvement beyond those works.

\end{remark}



The above two lemmas together imply the following guarantee for VB posterior:
\begin{thm}\label{mainthm}
Let $\sigma_0^2$ be a constant and $-\log\lambda =\log(T)+\delta[(L+1)\log N+\log\sqrt np]$ for any constant $\delta>0$. Under Conditions \ref{cond:width}-\ref{cond:activate}, \ref{cond:sieve}, 
we have with high probability $$\int_{\Theta} d^2(P_{\theta}, P_0)\widehat{q}(\theta)d\theta \leq C\varepsilon^{*2}_n + C'(r_n^* +\xi^*_n),$$ where $C$ is some positive constant and $C'$ is any diverging sequence.
\end{thm}

The $\varepsilon_n^{*2}$ denotes the estimation error from the statistical estimator for $P_0$. The variational Bayes convergence rate consists of estimating error, i.e., $\varepsilon_n^{*2}$, variational error, i.e., $r_n^*$, and approximation error, i.e., $\xi_n^*$. Given that the former two errors have only logarithmic difference, our convergence rate actually strikes the balance among all three error terms. The derived convergence rate has an explicit expression in terms of the network structure based on the forms of $\varepsilon_n^*$, $r_n^*$ and $\xi_n^*$, in contrast with general convergence results in \cite{Pati2018on,Zhang2019Convergence, Yang2020alpha}.

\begin{remark}
Theorem \ref{mainthm} provides a specific choice for $\lambda$, which can be relaxed to the general conditions on $\lambda$ in Lemma \ref{lmbound1}. In contrast to the heuristic choices such as $\lambda = \exp(-2\log n)$ \citep[BIC;][]{Hubin2019Combining}, this theoretically justified choice incorporates knowledge of input dimension, network structure and sample size. Such an $\lambda$ will be used in our numerical experiments in Section \ref{simu}, but readers shall be aware of that its theoretical validity is only justified in an asymptotic sense.
\end{remark}

\begin{remark}
The convergence rate is derived under Hellinger metric, which is of less practical relevance than $L_2$ norm representing the common prediction error. One may obtain a convergence result under $L_2$ norm via a VB truncation (refer to supplementary material, Theorem A.1).
\end{remark}

\begin{remark}
If $f_0$ is an $\alpha$-H{\"o}lder-smooth function with fixed input dimension $p$, then by choosing some $L\asymp\log n$, $N\asymp n/\log n$, combining with the approximation result \citep[Theorem 3]{Schmidt-Hieber2017Nonparametric}, our theorem ensures rate-minimax convergence up to a logarithmic term.
\end{remark}

\section{Implementation}
To conduct optimization of (\ref{eqvb3}) via stochastic gradient optimization, we need to find certain reparameterization for any distribution in $\mathcal Q$.
One solution is to use the inverse CDF sampling technique. Specifically, if $\theta\sim q \in \mathcal Q$, its marginal $\theta_i$'s are independent mixture of (\ref{vbpost}). Let $F_{(\mu_i,\sigma_i,\phi_i)}$ be the CDF of $\theta_i$, then $\theta_i \stackrel{d}{=} F^{-1}_{(\mu_i,\sigma_i,\phi_i)}(u_i)$ holds where $u_i \sim \mathcal{U}(0, 1)$. This inverse CDF reparameterization, although valid, can not be conveniently implemented within the state-of-art python packages like PyTorch. Rather, a more popular way in VB is to utilize the Gumbel-softmax approximation.

We rewrite the loss function $\Omega$ as
\begin{equation} \label{eq:loss}
\begin{split}
    -\mathbb{E}_{q(\theta|\gamma)q(\gamma)}[\log p_{\theta}(D)] + \sum^T_{i=1}\mbox{KL}(q(\gamma_i)||\pi(\gamma_i)) + \sum^T_{i=1}q(\gamma_i=1)\mbox{KL}(\mathcal{N}(\mu_i, \sigma^2_i)||\mathcal{N}(0, \sigma^2_0)).
\end{split}
\end{equation}
Since it is impossible to reparameterize the discrete variable $\gamma$ by a continuous system, we apply continuous relaxation, i.e., to approximate $\gamma$ by a continuous distribution. In particular, the Gumbel-softmax approximation \citep{Maddison17Concrete, Jang17Categorical} is used here, and $\gamma_i\sim \mbox{Bern}(\phi_i)$ is approximated by $\widetilde{\gamma}_i\sim\mbox{Gumbel-softmax}(\phi_i, \tau)$, where
\[
\begin{split}
    &\widetilde{\gamma}_i  =(1+\exp(-\eta_i/ \tau))^{-1}, \quad \eta_i = \log \frac{\phi_i}{1 - \phi_i} + \log \frac{u_i}{1-u_i}, \quad u_i \sim \mathcal{U}(0,1).
\end{split}
\]
$\tau$ is called the temperature, and as it approaches 0, $\tilde{\gamma}_i$ converges to $\gamma_i$ in distribution (refer to Figure 4 in the supplementary material). In addition, one can show that $P(\widetilde{\gamma}_i > 0.5) = \phi_i$, which implies
$\gamma_i \stackrel{d}{=} 1(\widetilde{\gamma}_i > 0.5)$. Thus, $\widetilde{\gamma}_i$ is viewed as a soft version of $\gamma_i$, 
and will be used in the backward pass to enable the calculation for gradients, while the hard version $\gamma_i$ will be used in the forward pass to obtain a sparse network structure.  
In practice, $\tau$ is usually chosen no smaller than 0.5 for numerical stability.
Besides, the normal variable $\mathcal{N}(\mu_i, \sigma^2_i)$ is reparameterized by $\mu_i + \sigma_i\epsilon_i$ for $\epsilon_i \sim \mathcal{N}(0, 1)$. 

The complete variational inference procedure with Gumbel-softmax approximation is stated below.
\begin{algorithm}[H]
\caption{Variational inference for sparse BNN with normal slab distribution.} \label{alg:vb}
\begin{algorithmic}[1]
\State {parameters: $\omega=(\mu, \sigma', \phi')$ }, \\
where $\sigma_i=\log(1+\exp(\sigma'_i))$, $\phi_i = (1+\exp(\phi'_i))^{-1}$, for $i=1,\ldots, T$
\Repeat
    \State {$D^m$ $\gets$ Randomly draw a minibatch of size $m$ from $D$}
    \State {$\epsilon_i, u_i$ $\gets$ Randomly draw $K$ samples from $\mathcal{N}(0,1)$ and  $\mathcal{U}(0,1)$} 
    \State {$\widetilde{\Omega}^m(\omega)$ $\gets$ Use (\ref{eq:sgvb}) with ($D^m$, $\omega$, $\epsilon$, $u$); Use $\gamma$ in the forward pass}
    \State {$\nabla_{\omega}\widetilde{\Omega}^m(\omega)$ $\gets$ Use (\ref{eq:sgvb}) with ($D^m$, $\omega$, $\epsilon$, $u$); Use $\widetilde{\gamma}$ in the backward pass }
    \State {$\omega$ $\gets$ Update with $\nabla_{\omega}\widetilde{\Omega}^m(\omega)$ using gradient descent algorithms (e.g. SGD or Adam)}
\Until{convergence of $\widetilde{\Omega}^m(\omega)$}
\State \Return {$\omega$}
\end{algorithmic}
\end{algorithm}
\vskip -0.1in

The Gumbel-softmax approximation introduces an additional error that may jeopardize the validity of Theorem \ref{mainthm}. Our exploratory studies (refer to Section B.3 in supplementary material) demonstrates little differences between the results of using inverse-CDF reparameterization and using Gumbel-softmax approximation in some simple model. Therefore, we conjecture that Gumbel-softmax approximation doesn't hurt the VB convergence, and thus will be implemented in our numerical studies.

\section{Experiments}\label{simu}
We evaluate the empirical performance of the proposed variational inference through simulation study and MNIST data application. For the simulation study, we consider a teacher-student framework and a nonlinear regression function, by which we justify the consistency of the proposed method and validate the proposed choice of hyperparameters.
As a byproduct, the performance of uncertainty quantification and the effectiveness of variable selection will be examined as well. 

For all the numerical studies, we let $\sigma^2_0=2$, the choice of $\lambda$ follows Theorem \ref{mainthm} (denoted by $\lambda_{opt}$): $\log(\lambda^{-1}_{opt})=\log(T)+0.1[(L+1)\log N+\log\sqrt np]$. The remaining details of implementation (such as initialization, choices of $K$, $m$ and learning rate) are provided in the supplementary material. We will use VB posterior mean estimator $\widehat{f}_H=\sum^{H}_{h=1}f_{\theta_h}/H$ to assess the prediction accuracy, where $\theta_h \sim \widehat{q}(\theta)$ are samples drawn from the VB posterior and $H=30$. The posterior network sparsity is measured by $\widehat{s} = \sum^T_{i=1}\phi_i/T$. Input nodes who have connection with $\phi_i>0.5$ to the second layer is selected as relevant input variables, and we report the corresponding false positive rate (FPR) and false negative rate (FNR) to evaluate the variable selection performance of our method.

Our method will be compared with the dense variational BNN (VBNN) \citep{Blundell2015weight} with independent centered normal prior and independent normal variational distribution, the AGP pruner \citep{Zhu2018Toprune}, the Lottery Ticket Hypothesis (LOT) \citep{Frankle2018the}, the variational dropout (VD) \citep{Molchanov2017Variational} and the Horseshoe BNN (HS-BNN) \citep{GHosh2018Structured}. In particular, VBNN can be regarded as a baseline method without any sparsification or compression.
All reported simulation results are based on 30 replications (except that we use 60 replications for interval estimation coverages). Note that the sparsity level in methods AGP and LOT are user-specified. Hence, in simulation studies, we try a grid search for AGP and LOT, and only report the ones that yield highest testing accuracy. Furthermore, note that FPR and FNR are not calculated for HS-BNN since it only sparsifies the hidden layers nodewisely.


\paragraph{Simulation I: Teacher-student networks setup} 
We consider two teacher network settings for $f_0$: (A) densely connected with a structure of 20-6-6-1, $p=20$, $n=3000$, $\sigma(x)=\mbox{sigmoid}(x)$, $X \sim \mathcal{U}([-1, 1]^{20})$, $\epsilon \sim \mathcal{N}(0, 1)$ and network parameter $\theta_i$ is randomly sampled from $\mathcal{U}(0, 1)$; (B) sparsely connected as shown in Figure \ref{fig:sparseteacher_lambda} (c), $p=100$, $n=500$, $\sigma(x)=\mbox{tanh}(x)$, $X \sim \mathcal{U}([-1, 1]^{100})$ and $\epsilon \sim \mathcal{N}(0, 1)$, the network parameter $\theta_i$'s are fixed (refer to supplementary material for details).

\begin{figure*}[t] 
    \centering
    \begin{subfigure}[b]{.3\textwidth}
        \centering
        \includegraphics[width=1\linewidth]{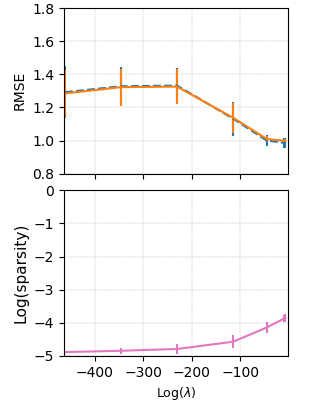}
        \caption{$\lambda \leq \lambda_{opt}$.}
    \end{subfigure}
        \begin{subfigure}[b]{.3\textwidth}
        \centering
        \includegraphics[width=1\linewidth]{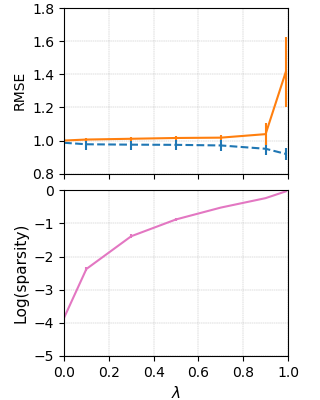}
        \caption{$\lambda \geq \lambda_{opt}$.}
    \end{subfigure}
    \begin{subfigure}[b]{.3\textwidth}
        \centering
        \includegraphics[width=1\linewidth]{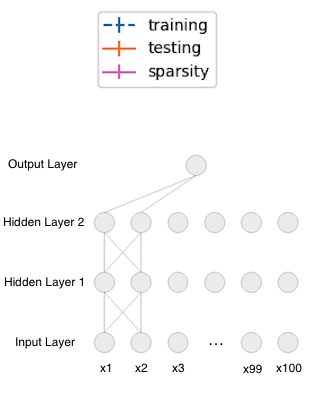}
        \caption{Sparse teacher network.}
    \end{subfigure}%
    \caption{(a) $\lambda= \{10^{-200}, 10^{-150}, 10^{-100}, 10^{-50}, 10^{-20}, 10^{-5}, \lambda_{opt}\}$. (b) $\lambda= \{\lambda_{opt}, 0.1, 0.3, 0.5, $ $0.7 , 0.9, 0.99\}$.  (c) The structure of the target sparse teacher network. Please note that the $x$ axes of figures (a) and (b) are in different scales.}
    \label{fig:sparseteacher_lambda}
\end{figure*}

\begin{table}[b]
\caption{Simulation results for Simulation I. SVBNN represents our sparse variational BNN method. The sparsity levels specified for AGP are 30\% and 5\%, and for LOT are 10\% and 5\%, respectively for the two cases.}
\label{tb:teacher}
\scriptsize
\centering
\begin{tabular}{rlcccccc}
\toprule
\multicolumn{1}{l}{} &  & \multicolumn{2}{c}{\textbf{RMSE}}  & \multicolumn{2}{c}{\textbf{Input variable selection}}  &  \\ \cmidrule(lr){3-4} \cmidrule(lr){5-6} 
\multicolumn{1}{l}{} & \textbf{Method} & \textbf{Train} & \textbf{Test} & \textbf{FPR(\%)} & \textbf{FNR(\%)}  & $\boldsymbol{95\%}$ \textbf{Coverage (\%)}  &\textbf{Sparsity(\%)} \\ \midrule
\multirow{6}{*}{\rotatebox[origin=c]{90}{Dense}} & SVBNN &1.01 $\pm$ 0.02 &1.01 $\pm$ 0.00   &- &-  &97.5 $\pm$ 1.71   &6.45 $\pm$ 0.83 \\
 & VBNN &1.00 $\pm$ 0.02 &1.00 $\pm$ 0.00 &- &-  &91.4 $\pm$ 3.89     &100 $\pm$ 0.00 \\
 & VD &0.99 $\pm$ 0.02 &1.01 $\pm$ 0.00  &-  &-  & 76.4 $\pm$ 4.75   &28.6 $\pm$ 2.81\\
 & HS-BNN &0.98 $\pm$ 0.02 &1.02 $\pm$ 0.01  & - & -   & 83.5 $\pm$ 0.78  &64.9 $\pm$ 24.9\\
 & AGP &0.99 $\pm$ 0.02 & 1.01 $\pm$ 0.00  & - & -   & -  &30.0 $\pm$ 0.00\\
 & LOT &1.04 $\pm$ 0.01  &1.02 $\pm$ 0.00   & - & -   & -  &10.0 $\pm$ 0.00\\
\midrule
\multirow{6}{*}{\rotatebox[origin=c]{90}{Sparse}} & SVBNN &0.99 $\pm$ 0.03 &1.00 $\pm$ 0.01  &0.00 $\pm$ 0.00 &0.00 $\pm$ 0.00  &96.4 $\pm$ 4.73   &2.15 $\pm$ 0.25 \\
 & VBNN &0.92 $\pm$ 0.05  &1.53 $\pm$ 0.17 &100 $\pm$ 0.00  &0.00 $\pm$ 0.00   &90.7 $\pm$ 8.15  &100 $\pm$ 0.00\\
 & VD &0.86 $\pm$ 0.04  &1.07 $\pm$ 0.03 &72.9 $\pm$ 6.99 &0.00 $\pm$ 0.00   & 75.5 $\pm$ 7.81  &20.8 $\pm$ 3.08 \\
 &HS-BNN &0.90 $\pm$ 0.04 &1.29 $\pm$ 0.04 &-&-&67.0 $\pm$ 8.54 &32.1 $\pm$ 20.1\\
 & AGP &1.01 $\pm$ 0.03  &1.02 $\pm$ 0.00  &16.9 $\pm$ 1.81  &0.00 $\pm$ 0.00   & -   & 5.00 $\pm$ 0.00 \\
 & LOT &0.96 $\pm$ 0.01  &1.04 $\pm$ 0.01  &19.5 $\pm$ 2.57  &0.00 $\pm$ 0.00    & -  & 5.00 $\pm$ 0.00\\ 
 \bottomrule
\end{tabular}
\end{table}

First, we examine the impact of different choices of $\lambda$ on our VB sparse DNN modeling. 
A set of different $\lambda$ values are used, and for each $\lambda$, we compute the training square root MSE (RMSE) and testing RMSE based on $\widehat{f}_H$. Results for the simulation setting (B) are plotted in Figure \ref{fig:sparseteacher_lambda} along with error bars (Refer to Section B.4 in supplementary material for the plot under the simulation setting (A)). The figure shows that as $\lambda$ increases, the resultant network becomes denser and the training RMSE monotonically decreases, while testing RMSE curve is roughly U-shaped. In other words, an overly small $\lambda$ leads to over-sparsified DNNs with insufficient expressive power, and an overly large $\lambda$ leads to overfitting DNNs. The suggested $\lambda_{opt}$ successfully locates in the valley of U-shaped testing curve, which empirically justifies our theoretical choice of $\lambda_{opt}$.


We next compare the performance of our method (with $\lambda_{opt}$) to the benchmark methods, and present results in Table $\ref{tb:teacher}$.
For the dense teacher network (A), our method leads to the most sparse structure with comparable prediction error; For the sparse teacher network (B), our method not only achieves the best prediction accuracy, but also always selects the correct set of relevant input variables.
Besides, we also explore uncertainty quantification of our methods, by studying the coverage of 95\% Bayesian predictive intervals (refer to supplementary material for details). 
Table $\ref{tb:teacher}$ shows that our method obtains coverage rates slightly higher than the nominal levels while other (Bayesian) methods suffer from undercoverage problems.

\paragraph{Simulation II: Sparse nonlinear function}
Consider the following sparse function $f_0$:
\begin{equation} \label{eq:sparsefunc}
    f_0(x_1,\dots,x_{200}) = \frac{7x_2}{1 + x^2_1} + 5\sin(x_3x_4)+2x_5, \quad \epsilon \sim \mathcal{N}(0,1),
\end{equation}
all covariates are iid $\mathcal{N}(0, 1)$ and data set contains $n=3000$ observations. A ReLU network with $L=3$ and $N=7$ is used. Similar to the simulation I, we study the impact of $\lambda$, and results in Figure $\ref{fig:sparsefunc_lambda}$ justify that $\lambda_{opt}$ is a reasonable choice. 
Table \ref{tb:sparsereg} compares the performances of our method (under $\lambda_{opt}$) to the competitive methods. Our method exhibits the best prediction power with minimal connectivity, among all the methods. In addition, our method achieves smallest FPR and acceptable FNR for input variable selection. In comparison, other methods select huge number of false input variables. Figure \ref{fig:sparsefunc_lambda} (c) shows the selected network (edges with $\phi_i>0.5$) in one replication that correctly identifies the input variables.

\begin{figure*}[tb] 
    \centering   
    \begin{subfigure}[b]{.3\textwidth}
        \centering
        \includegraphics[width=1\linewidth]{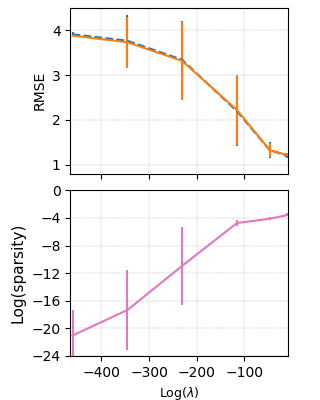}
        \caption{$\lambda \leq \lambda_{opt}$.}
    \end{subfigure}
        \begin{subfigure}[b]{.3\textwidth}
        \centering
        \includegraphics[width=1\linewidth]{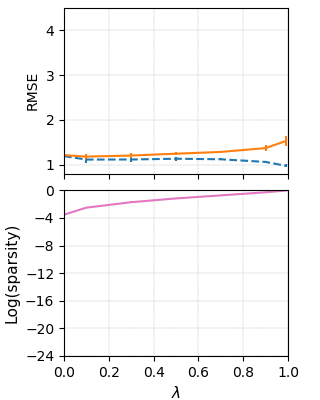}
        \caption{$\lambda \geq \lambda_{opt}$.}
    \end{subfigure}
    \begin{subfigure}[b]{.3\textwidth}
        \centering
        \includegraphics[width=1\linewidth]{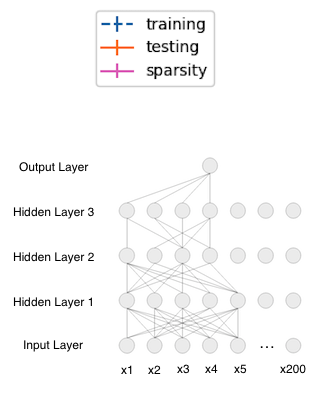}
        \caption{Selected network structure.}
    \end{subfigure}%
    \caption{(a) $\lambda = \{10^{-200}, 10^{-150}, 10^{-100}, 10^{-50}, 10^{-20}, 10^{-5}, \lambda_{opt}\}$. (b) $\lambda = \{\lambda_{opt}, 0.1, 0.3, 0.5,$ $ 0.7, 0.9, 0.99\}$.  (c) A selected network structure for (\ref{eq:sparsefunc}). 
    }
    \label{fig:sparsefunc_lambda}
\end{figure*}

\begin{table}[htb]
\caption{Results for Simulation II. The sparsity levels selected for AGP and LOT are both 30\%.}
\label{tb:sparsereg}
\centering
\begin{tabular}{lccccc}
\toprule
\textbf{Method} & \textbf{Train RMSE} & \textbf{Test RMSE} &\textbf{FPR(\%)} &\textbf{FNR(\%)} & \textbf{Sparsity(\%)}  \\ 
\midrule
SVBNN &1.19 $\pm$ 0.05 &1.21 $\pm$ 0.05  &0.00 $\pm$ 0.21 & 16.0 $\pm$ 8.14 &2.97 $\pm$ 0.48  \\
VBNN  &0.96 $\pm$ 0.06 &1.99 $\pm$ 0.49  &100 $\pm$ 0.00  &0.00 $\pm$ 0.00 &100  $\pm$ 0.00 \\
VD    &1.02 $\pm$ 0.05 &1.43 $\pm$ 0.19  &98.6 $\pm$ 1.22 &0.67 $\pm$ 3.65 &46.9 $\pm$ 4.72 \\
HS-BNN &1.17 $\pm$ 0.52&1.66 $\pm$ 0.43&-&-&41.1 $\pm$ 36.5\\
AGP   &1.06 $\pm$ 0.08 &1.58 $\pm$ 0.11  &82.7 $\pm$ 3.09  &1.33 $\pm$ 5.07 &30.0 $\pm$ 0.00   \\
LOT   &1.08 $\pm$ 0.09 &1.44 $\pm$ 0.14  &83.6 $\pm$ 2.94 &0.00 $\pm$ 0.00 &30.0 $\pm$ 0.00\\
\bottomrule
\end{tabular}
\end{table}

\begin{wrapfigure}{R}{0.5\textwidth}
    \centering
    \includegraphics[width=0.42\textwidth]{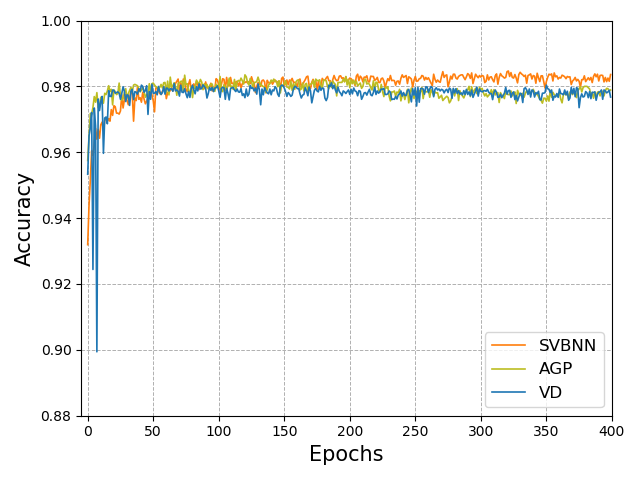}
    \caption{Testing accuracy for MNIST} 
    \label{fig:MNIST}
\end{wrapfigure} 
\paragraph{MNIST application.}

We evaluate the performance of our method on MNIST data for classification tasks, by comparing with benchmark methods. A 2-hidden layer DNN with 512 neurons in each layer is used. We compare the testing accuracy of our method (with $\lambda_{opt}$) to the benchmark methods at different epochs using the same batch size (refer to supplementary material for details). Figure $\ref{fig:MNIST}$ shows our method achieves best accuracy as epoch increases, and the final sparsity level for SVBNN, AGP and VD are $5.06\%$, $5.00\%$ and $2.28\%$. 

In addition, an illustration of our method's capability for uncertainty quantification on MNIST can be found in the supplementary material, where additional experimental results on UCI regression datasets can also be found.

\section{Conclusion and discussion}
We proposed a variational inference method for deep neural networks under spike-and-slab priors with theoretical guarantees. Future direction could be investigating the theory behind choosing hyperparamters via the EB estimation instead of deterministic choices.

Furthermore, extending the current results to more complicated networks (convolutional neural network, residual network, etc.) is not trivial. Conceptually, it requires the design of structured sparsity (e.g., group sparsity in \cite{Neklyudov2017Structured}) to fulfill the goal of faster prediction. Theoretically, it requires deeper understanding of the expressive ability (i.e. approximation error) and capacity (i.e., packing or covering number) of the network model space. For illustration purpose, we include an example of Fashion-MNIST task using convolutional neural network in the supplementary material, and it demonstrates the usage of our method on more complex networks in practice.

\section*{Broader Impact}
We believe the ethical aspects are not applicable to this work.
For future societal consequences, deep learning has a wide range of applications such as computer version and natural language processing. Our work provides a solution to overcome the drawbacks of modern deep neural network, and also improves the understanding of deep learning.

The proposed method could improve the existing applications. Specifically, sparse learning helps apply deep neural networks to hardware limited devices, like cell phones or pads, which will broaden the horizon of deep learning application. In addition, as a Bayesian method, not only a result, but also the knowledge of confidence or certainty in that result are provided, which could benefit people in various aspects. For example, in the application of cancer diagnostic, by providing the certainty associated with each possible outcome, Bayesian learning would assist the medical professionals to make a better judgement about whether the tumor is a cancer or a benign one. Such kind of ability to quantify uncertainty would contribute to the modern deep learning.

\begin{ack}
We would like to thank Wei Deng for the helpful discussion and thank the reviewers for their thoughtful comments. Qifan Song's research is partially supported by National Science Foundation grant DMS-1811812. Guang Cheng was a member of the Institute for Advanced Study in writing this paper, and he would like to thank the institute for its hospitality.

\end{ack}

\bibliography{ref}

\begin{thebibliography}{}

\bibitem[Blei et~al., 2017]{Blei2017variational}
Blei, D., Kucukelbir, A., and McAuliffe, J. (2017).
\newblock Variational inference: A review for statisticians.
\newblock {\em Journal of the American Statistical Association}, 112:859--877.

\bibitem[Blundell et~al., 2015]{Blundell2015weight}
Blundell, C., Cornebise, J., Kavukcuoglu, K., et~al. (2015).
\newblock Weight uncertainty in neural networks.
\newblock In {\em Proceedings of the 32nd International Conference on
  International Conference on Machine Learning (ICML 15)}, pages 1613--1622,
  Lille, France.

\bibitem[Boucheron et~al., 2013]{Boucheron2013Concentration}
Boucheron, S., Lugosi, G., and Massart, P. (2013).
\newblock {\em Concentration inequalities: A nonasymptotic theory of
  independence}.
\newblock Oxford University press.

\bibitem[Cheng et~al., 2018]{Cheng2018Model}
Cheng, Y., Wang, D., Zhou, P., et~al. (2018).
\newblock Model compression and acceleration for deep neural networks: The
  principles, progress, and challenges.
\newblock {\em IEEE Signal Processing Magazine}, 35(1):126--136.

\bibitem[Ch{\'e}rief-Abdellatif, 2020]{Cherief2020convergence}
Ch{\'e}rief-Abdellatif, B.-E. (2020).
\newblock Convergence rates of variational inference in sparse deep learning.
\newblock In {\em Proceedings of the 37th International Conference on Machine
  Learning (ICML 2020)}, Vienna, Austria.

\bibitem[Ch{\'e}rief-Abdellatif and Alquier, 2018]{cherief2018consistency}
Ch{\'e}rief-Abdellatif, B.-E. and Alquier, P. (2018).
\newblock Consistency of variational bayes inference for estimation and model
  selection in mixtures.
\newblock {\em Electronic Journal of Statistics}, 12(2):2995--3035.

\bibitem[Deng et~al., 2019]{Deng2019Adaptive}
Deng, W., Zhang, X., Liang, F., and Lin, G. (2019).
\newblock An adaptive empirical {Bayesian} method for sparse deep learning.
\newblock In {\em 33rd Conference on Neural Information Processing Systems
  (NeurIPS 2019)}, Vancouver, Canada.

\bibitem[Feng and Simon, 2017]{Feng2017Sparse}
Feng, J. and Simon, N. (2017).
\newblock Sparse input neural networks for high-dimensional nonparametric
  regression and classification.
\newblock {\em arXiv preprint arXiv:1711.07592}.

\bibitem[Frankle and Carbin, 2018]{Frankle2018the}
Frankle, J. and Carbin, M. (2018).
\newblock The lottery ticket hypothesis: Finding sparse, trainable neural
  networks.
\newblock {\em arXiv preprint arXiv:1803.03635}.

\bibitem[Gale et~al., 2019]{Gale2019state}
Gale, T., Elsen, E., and Hooker, S. (2019).
\newblock The state of sparsity in deep neural networks.
\newblock {\em arXiv preprint arXiv:1902.09574}.

\bibitem[George and McCulloch, 1993]{George1993variable}
George, E. and McCulloch, R. (1993).
\newblock Variable selection via gibbs sampling.
\newblock {\em Journal of the American Statistical Association}, 88:881--889.

\bibitem[Ghosal and Van Der~Vaart, 2007]{ghosal2007convergence}
Ghosal, S. and Van Der~Vaart, A. (2007).
\newblock Convergence rates of posterior distributions for noniid observations.
\newblock {\em The Annals of Statistics}, 35(1):192--223.

\bibitem[Ghosh and Doshi-Velez, 2017]{ghosh2017Model}
Ghosh, S. and Doshi-Velez, F. (2017).
\newblock Model selection in {Bayesian} neural networks via horseshoe priors.
\newblock {\em arXiv preprint arXiv:1705.10388}.

\bibitem[Ghosh et~al., 2018]{GHosh2018Structured}
Ghosh, S., Yao, J., and Doshi-Velez, F. (2018).
\newblock Structured variational learning of {Bayesian} neural networks with
  horseshoe priors.
\newblock In {\em Proceedings of the 35th International Conference on Machine
  Learning (ICML 2018)}, Stockholm, Sweden.

\bibitem[Goldt et~al., 2019]{Goldt2019Dynamics}
Goldt, S., Advani, M.~S., Saxe, A.~M., Krzakala, F., and Zdeborová, L. (2019).
\newblock Dynamics of stochastic gradient descent for two-layer neural networks
  in the teacher-student setup.
\newblock In {\em 33rd Conference on Neural Information Processing Systems
  (NeurIPS 2019)}, Vancouver, Canada.

\bibitem[Guo et~al., 2018]{Guo2018Sparse}
Guo, Y., Zhang, C., Zhang, C., and Chen, Y. (2018).
\newblock Sparse {DNNs} with improved adversarial robustness.
\newblock In {\em 32nd Conference on Neural Information Processing Systems
  (NeurIPS 2018)}, pages 240--249, Montréal, Canada.

\bibitem[Han et~al., 2016]{Han2016Deep}
Han, S., Mao, H., and Dally, W. (2016).
\newblock Deep compression: Compressing deep neural networks with pruning,
  trained quantization and huffman coding.
\newblock In {\em International Conference on Learning Representations (ICLR)}.

\bibitem[Hern\'{a}ndez-Lobato and Adams, 2015]{Lobato2015Probabilistic}
Hern\'{a}ndez-Lobato, J. and Adams, R. (2015).
\newblock Probabilistic backpropagation for scalable learning of bayesian
  neural networks.
\newblock In {\em Proceedings of the 32nd International Conference on Machine
  Learning (ICML 2015)}, Lille, France.

\bibitem[Hubin and Storvik, 2019]{Hubin2019Combining}
Hubin, A. and Storvik, G. (2019).
\newblock Combining model and parameter uncertainty in bayesian neural
  networks.
\newblock {\em arXiv:1903.07594}.

\bibitem[Ishwaran and Rao, 2005]{Ishwaran2005Spike}
Ishwaran, H. and Rao, S. (2005).
\newblock Spike and slab variable selection: Frequentist and bayesian
  strategies.
\newblock {\em Annals of Statistics}, 33(2):730--773.

\bibitem[Jang et~al., 2017]{Jang17Categorical}
Jang, E., Gu, S., and Poole, B. (2017).
\newblock Categorical reparameterization with gumbel-softmax.
\newblock In {\em International Conference on Learning Representations (ICLR
  2017)}.

\bibitem[Jordan et~al., 1999]{Jordan1999Variational}
Jordan, M., Ghahramani, Z., Jaakkola, T., et~al. (1999).
\newblock An introduction to variational methods for graphical models.
\newblock {\em Machine Learning}.

\bibitem[Kingma and Welling, 2014]{Kingma2014VAE}
Kingma, D. and Welling, M. (2014).
\newblock Auto-encoding variational {Bayes}.
\newblock {\em arXiv:1312.6114}.

\bibitem[Le~Cam, 1986]{le2012asymptotic}
Le~Cam, L. (1986).
\newblock {\em Asymptotic methods in statistical decision theory}.
\newblock Springer Science \& Business Media, New York.

\bibitem[Liang et~al., 2018]{Liang2018Bayesian}
Liang, F., Li, Q., and Zhou, L. (2018).
\newblock Bayesian neural networks for selection of drug sensitive genes.
\newblock {\em Journal of the American Statistical Association}, 113:955--972.

\bibitem[Louizos et~al., 2018]{Louizos2018learning}
Louizos, C., Welling, M., and Kingma, D.~P. (2018).
\newblock Learning sparse neural networks through l0 regularization.
\newblock In {\em ICLR 2018}.

\bibitem[MacKay, 1992]{Mackay1992practical}
MacKay, D. (1992).
\newblock A practical bayesian framework for backpropagation networks.
\newblock {\em Nerual Computation}.

\bibitem[Maddison et~al., 2017]{Maddison17Concrete}
Maddison, C., Mnih, A., and Teh, Y.~W. (2017).
\newblock The concrete distribution: A continuous relaxation of discrete random
  variables.
\newblock In {\em International Conference on Learning Representations (ICLR
  2017)}.

\bibitem[Mocanu et~al., 2018]{Mocanu2018scalable}
Mocanu, D., Mocanu, E., Stone, P., et~al. (2018).
\newblock Scalable training of artificial neural networks with adaptive sparse
  connectivity inspired by network science.
\newblock {\em Nature Communications}, 9:2383.

\bibitem[Molchanov et~al., 2017]{Molchanov2017Variational}
Molchanov, D., Ashukha, A., and Vetrov, D. (2017).
\newblock Variational dropout sparsifies deep neural networks.
\newblock In {\em Proceedings of the 34th International Conference on Machine
  Learning (ICML 2017)}, pages 2498--2507.

\bibitem[Neal, 1992]{Neal1993Bayesian}
Neal, R. (1992).
\newblock Bayesian learning via stochastic dynamics.
\newblock In {\em Advances in Neural Information Processing Systems 5 (NIPS
  1992)}, pages 475--482.

\bibitem[Neklyudov et~al., 2017]{Neklyudov2017Structured}
Neklyudov, K., Molchanov, D., Ashukha, A., and Vetrov, D. (2017).
\newblock Structured {Bayesian} pruning via log-normal multiplicative noise.
\newblock In {\em 31st Conference on Neural Information Processing Systems
  (NIPS 2017)}, Long Beach, CA.

\bibitem[Pati et~al., 2018]{Pati2018on}
Pati, D., Bhattacharya, A., and Yang, Y. (2018).
\newblock On the statistical optimality of variational bayes.
\newblock In {\em Proceedings of the 21st International Conference on
  Artificial Intelligence and Statistics (AISTATS) 2018}, Lanzarote, Spain.

\bibitem[Polson and Rockova, 2018]{Polson2018posterior}
Polson, N. and Rockova, V. (2018).
\newblock Posterior concentration for sparse deep learning.
\newblock In {\em 32nd Conference on Neural Information Processing Systems
  (NeurIPS 2018)}, pages 930--941, Montréal, Canada.

\bibitem[Schmidt-Hieber, 2017]{Schmidt-Hieber2017Nonparametric}
Schmidt-Hieber, J. (2017).
\newblock Nonparametric regression using deep neural networks with {ReLU}
  activation function.
\newblock {\em arXiv:1708.06633v2}.

\bibitem[S{\o}nderby et~al., 2016]{sonderby2016How}
S{\o}nderby, C., Raiko, T., Maal{\o}e, L., S{\o}nderby, S., and Ole, W. (2016).
\newblock How to train deep variational autoencoders and probabilistic ladder
  networks.
\newblock In {\em Proceedings of the 33rd International Conference on
  International Conference on Machine Learning (ICML 16)}, New York, NY.

\bibitem[Tian, 2018]{tian2018theoretical}
Tian, Y. (2018).
\newblock A theoretical framework for deep locally connected {ReLU} network.
\newblock {\em arXiv preprint arXiv:1809.10829}.

\bibitem[Yang et~al., 2020]{Yang2020alpha}
Yang, Y., Pati, D., and Bhattacharya, A. (2020).
\newblock $\alpha$-variational inference with statistical guarantees.
\newblock {\em Annals of Statistics}, 48(2):886--905.

\bibitem[Ye and Sun, 2018]{Ye2018variable}
Ye, M. and Sun, Y. (2018).
\newblock Variable selection via penalized neural network: a drop-out-one loss
  approach.
\newblock In {\em Proceedings of the 35th International Conference on
  International Conference on Machine Learning (ICML 18)}, pages 5620--5629,
  Stockholm, Sweden.

\bibitem[Zhang and Gao, 2019]{Zhang2019Convergence}
Zhang, F. and Gao, C. (2019).
\newblock Convergence rates of variational posterior distributions.
\newblock {\em arXiv preprint arXiv:1712.02519}.

\bibitem[Zhu and Gupta, 2018]{Zhu2018Toprune}
Zhu, M. and Gupta, S. (2018).
\newblock To prune, or not to prune: Exploring the efficacy of pruning for
  model compression.
\newblock In {\em International Conference on Learning Representations (ICLR)}.

\end{thebibliography}

\newpage
\begin{center}
	\textbf{\LARGE Supplementary Document to the Paper "Efficient Variational Inference for Sparse Deep Learning with Theoretical Guarantee"}
\end{center}
\appendix
\vspace{30px}

In this document, the detailed proofs for the theoretical results are provided in the first section, along with additional numerical results presented in the second section.

\section{Proofs of theoretical results}

\subsection{Proof of Lemma 4.1}
As a technical tool for the proof, we first restate the Lemma 6.1 in \cite{cherief2018consistency} as follows.
\begin{lem} \label{lem:mixture}
	For any $K>0$, the KL divergence between any two mixture densities $\sum_{k=1}^Kw_kg_k$ and $\sum_{k=1}^K\tilde{w}_k\tilde{g}_k$ is bounded as
	\[
	\mbox{KL}(\sum_{k=1}^Kw_kg_k || \sum_{k=1}^K\tilde{w}_k\tilde{g}_k) \leq \mbox{KL}(\boldsymbol{w}||\tilde{\boldsymbol{w}}) + \sum^K_{k=1}w_k\mbox{KL}(g_k||\tilde{g}_k),
	\]
	where $\mbox{KL}(\boldsymbol{w}||\tilde{\boldsymbol{w}}) = \sum_{k=1}^Kw_k\log \frac{w_k}{\Tilde{w}_k}$.
\end{lem}

{\noindent\bf Proof of Lemma 4.1}

\begin{proof}
	It suffices to construct some $q^*(\theta) \in \mathcal Q$, such that w.h.p,
	\[
	\begin{split}
	&\mbox{KL}(q^*(\theta)||\pi(\theta|\lambda)) + \int_{\Theta} l_n(P_0, P_{\theta})q^*(\theta)(d\theta)\\
	\leq &C_1nr^*_n+C'_1n\inf_{\theta}||f_{\theta}-f_0||^2_{\infty}+C'_1nr^*_n,
	\end{split}
	\]
	where $C_1$, $C'_1$ are some positive constants if $\lim n(r_n^*+\xi^*_n)=\infty$, or any diverging sequences if $\lim\sup n(r_n^*+\xi^*_n)\neq \infty$.
	
	Recall $\theta^*=\arg\min_{\theta \in \Theta(L,\boldsymbol{p},s^*,B)} ||f_{\theta}-f_0||^2_{\infty}$, then $q^*(\theta) \in \mathcal Q$ can be constructed as
	\begin{align}
	&\mbox{KL}(q^*(\theta)||\pi(\theta|\lambda)) \leq C_1nr^*_n \label{eqche2},\\
	&\int_{\Theta} ||f_{\theta} - f_{\theta^*}||_{\infty}^2q^*(\theta)(d\theta) \leq r^*_n. \label{eqche1}
	\end{align}
	
	We define $q^*(\theta)$ as follows, for $i=1,\dots, T$:
	\begin{equation}\label{vbstar}
	\begin{split}
	& \theta_i|\gamma^*_i \sim \gamma^*_i\mathcal{N}(\theta_i^*,\sigma^2_n)+(1-\gamma^*_i)\delta_{0},\\
	& \gamma_i^* \sim \mbox{Bern}(\phi^*_i), \\
	& \phi^*_i = 1(\theta^*_i\neq 0),
	\end{split}
	\end{equation}
	where $a_n^2=\frac{s^*}{8n}\log^{-1}(3pN)(2BN)^{-2(L+1)}\Bigl\{(p + 1 + \frac{1}{BN-1})^2 + \frac{1}{(2BN)^2-1}
	+\frac{2}{(2BN-1)^2}\Bigr\}^{-1}$.

	
	To prove (\ref{eqche2}), denote $\Gamma^T$ as the set of all possible binary inclusion vectors with length $T$, then $q^*(\theta)$ and $\pi(\theta|\lambda)$ could be written as mixtures
	\[
	q^*(\theta) = \sum_{\gamma \in \Gamma^T}1(\gamma = \gamma^*)\prod^{T}_{i=1}\gamma_i\mathcal{N}(\theta^*_i, \sigma^2_n) + (1 - \gamma_i)\delta_{0},
	\]
	and 
	\[
	\pi(\theta|\lambda) =\sum_{\gamma \in \Gamma^T} \pi(\gamma)\prod^T_{i=1}\gamma_i \mathcal{N}(0, \sigma^2_0) + (1 - \gamma_i)\delta_{0},
	\]
	where $\pi(\gamma)$ is the probability for vector $\gamma$ under prior distribution $\pi$. Then,
	\[
	\begin{split}
	&\mbox{KL}(q^*(\theta)||\pi(\theta|\lambda))\\
	\leq &\log\frac{1}{\pi(\gamma^*)} + \sum_{\gamma \in \Gamma^T}1(\gamma = \gamma^*)\mbox{KL}\Bigl\{\prod^{T}_{i=1}\gamma_i\mathcal{N}(\theta^*_i, \sigma^2_n) + (1 - \gamma_i)\delta_{0}\Bigl|\Bigr| \prod^T_{i=1}\gamma_i \mathcal{N}(0, \sigma^2_0)) + (1 - \gamma_i)\delta_{0}\Bigr\}\\
	= &\log \frac{1}{\lambda^{s^*} (1-\lambda)^{T-s^*}} + \sum^T_{i=1}\mbox{KL}\Bigl\{\gamma^*_i\mathcal{N}(\theta^*_i, \sigma^2_n) + (1 - \gamma^*_i)\delta_{0}|| \gamma^*_i \mathcal{N}(0, \sigma^2_0)) + (1 - \gamma^*_i)\delta_{0}\Bigr\}\\
	= &s^*\log(\frac{1}{\lambda})+(T-s^*)\log(\frac{1}{1-\lambda})+ \sum^T_{i=1}\gamma^*_i\Bigl\{\frac{1}{2}\log\Bigl(\frac{\sigma^2_0}{\sigma^2_n}\Bigr)+\frac{\sigma^2_n+\theta^{*2}_i}{2} - \frac{1}{2}\Bigr\}\\
	\leq &C_0nr^*_n  + \frac{s^*}{2}\sigma^2_n+\frac{s^*}{2}(B^2-1) +\frac{s^*}{2}\log\Bigl(\frac{\sigma^2_0}{\sigma^2_n}\Bigr)\\
	\leq &(C_0+1)nr^*_n+\frac{s^*}{2}B^2+\frac{s^*}{2}\log\Bigl(\frac{8n}{s^*}\log(3pN)(2BN)^{2L+2}\Bigl\{(p+1+\frac{1}{BN-1})^2\\
	&+\frac{1}{(2BN)^2-1}+\frac{2}{(2BN-1)^2}\Bigr\}\Bigr)\\
	\leq &(C_0+2)nr^*_n + \frac{B^2}{2}s^* + (L+1)s^*\log(2BN) + \frac{s^*}{2}\log\log(3BN)+\frac{s^*}{2}\log\Bigl(\frac{n}{s^*}p^2\Bigr)\\
	\leq &(C_0+3)nr^*_n + (L+1)s^*\log N+s^*\log\Bigl(p\sqrt{\frac{n}{s^*}}\Bigr)\\
	\leq & C_1nr^*_n ,  \mbox{ for sufficiently large n},
	\end{split}
	\]
	where $C_0$ and $C_1$ are some fixed constants. The first inequality is due to Lemma \ref{lem:mixture} and the second inequality is due to Condition $4.4$.
	Furthermore, by Appendix G of \cite{Cherief2020convergence}, it can be shown

	\[
	\begin{split}
	&\int_{\Theta} ||f_{\theta}-f_{\theta^*}||^2_{\infty}q^*(\theta)(d\theta) \\
	\leq &8a^2_n\log(3BN)(2BN)^{2L+2}\Bigl\{(p + 1 + \frac{1}{BN-1})^2 + \frac{1}{(2BN)^2-1}
	+\frac{2}{(2BN-1)^2}\Bigr\}\\
	\leq &\frac{s^*}{n} \leq r^*_n.
	\end{split}
	\]
	
	
	Noting that 
	\[
	\begin{split}
	l_n(P_0, P_{\theta}) &= \frac{1}{2\sigma^2_{\epsilon}}(||Y - f_{\theta}(X)||^2_2 - ||Y - f_{0}(X)||^2_2)\\
	&= \frac{1}{2\sigma^2_{\epsilon}} (||Y - f_{0}(X) + f_0(X)-f_{\theta}(X))||^2_2 - ||Y - f_{0}(X)||^2_2)\\
	&= \frac{1}{2\sigma^2_{\epsilon}}(||f_{\theta}(X)-f_0(X)||^2_2 + 2\langle Y-f_0(X), f_0(X) - f_\theta(X)\rangle),
	\end{split}
	\]
	
	Denote
	\[
	\begin{split}
	\mathcal{R}_1 &= \int_{\Theta} ||f_{\theta}(X) - f_0(X)||^2_2 q^*(\theta)(d\theta), \\
	\mathcal{R}_2 &=\int_{\Theta} \langle Y-f_0(X), f_0(X) - f_\theta(X)\rangle q^*(\theta)(d\theta).
	\end{split}
	\]

	Since $||f_{\theta}(X) - f_0(X)||^2_2 \leq n ||f_{\theta} - f_0||^2_{\infty}\leq n||f_{\theta}-f_{\theta^{\ast}}||^2_{
		\infty}+n||f_{\theta^{\ast}}-f_0||^2_{\infty}$,
	\[
	\mathcal{R}_1 \leq  nr^*_n + n||f_{\theta^{\ast}}-f_0||^2_{\infty}. 
	\]
	
	Noting that $Y - f_0(X) = \epsilon \sim \mathcal{N}(0, \sigma^2_{\epsilon}I)$, then
	\[
	\begin{split}
	\mathcal{R}_2 &= \int_{\Theta} \epsilon^T(f_0(X) - f_\theta(X))q^*(\theta)(d\theta)= \epsilon^T\int_{\Theta} (f_0(X) - f_\theta(X))q^*(\theta)(d\theta) \sim \mathcal{N}(0, c_f\sigma^2_{\epsilon}),
	\end{split}
	\]
	where $c_f = ||\int_{\Theta} (f_0(X) - f_\theta(X))q^*(\theta)(d\theta)||^2_2 \leq \mathcal{R}_1$ due to Cauchy-Schwarz inequality. Therefore, 
	$\mathcal{R}_2=O_p(\sqrt{\mathcal{R}_1})$, and w.h.p., $\mathcal{R}_2\leq C'_0\mathcal{R}_1$, where $C'_0$ is some positive constant if $\lim n(r_n^*+\xi_n^*)=\infty$ or $C'_0$ is any diverging sequence if $\lim\sup  n(r_n^*+\xi_n^*)\neq \infty$.
	Therefore,
	\[
	\begin{split}
	\int_{\Theta} l_n(P_0, P_{\theta})q^*(\theta)(d\theta)= \mathcal{R}_1/2\sigma^2_{\epsilon} + \mathcal{R}_2/\sigma^2_{\epsilon} \leq &(2C'_0+1)n (r^*_n + ||f_{\theta^{\ast}}-f_0||^2_{\infty})/2\sigma_{\varepsilon}^2\\
	\leq & C'_1(nr^*_n + ||f_{\theta^{\ast}}-f_0||^2_{\infty}))\mbox{,  w.h.p.}, 
	\end{split}
	\]
	which concludes this lemma together with (\ref{eqche2}).
\end{proof}

\subsection{Proof of Lemma 4.2}

Under Condition 4.1 - 4.2, we have the following lemma that shows the existence of testing functions over $\Theta_n=\Theta(L, \boldsymbol{p}, s_n)$, where $\Theta(L, \boldsymbol{p}, s_n)$ denotes the set of parameter whose $L_0$ norm is bounded by $s_n$.
\begin{lem}{\label{lmtesting}}
	Let $\varepsilon^*_n=Mn^{-1/2}\sqrt{(L+1)s^*\log N + s^*\log(p\sqrt{n/s^*})}\log^\delta(n)$ for any $\delta>1$ and some large constant M. 
	Let $s_n = s^*\log^{2\delta-1} n$.
	Then there exists some testing function $\phi\in[0,1]$ and $C_1>0$, $C_2>1/3$, such that
	\[
	\begin{split}
	\mathbb{E}_{P_0}(\phi)&\leq  \exp\{-C_1n\varepsilon_n^{*2}\},\\
	\sup_{\substack{P_{\theta} \in \mathcal{F}(L,\boldsymbol{p},s_n)\\ d(P_{\theta}, P_0)>\varepsilon^*_n}}\mathbb{E}_{P_{\theta}}(1-\phi)&\leq\exp \{-C_2nd^2(P_0,P_{\theta})\}.
	\end{split}
	\]
\end{lem}
\begin{proof}
	Due to the well-known result (e.g., \cite{le2012asymptotic}, page 491 or \cite{ghosal2007convergence}, Lemma 2), there always exists a function $\psi\in[0,1]$, such that
	\[
	\begin{split}
	&\mathbb{E}_{P_0}(\psi)\leq \exp\{-nd^2(P_{\theta_1},P_{0}) /2\},\\
	&\mathbb{E}_{P_{\theta}}(1-\psi)\leq \exp\{-nd^2(P_{\theta_1},P_{0}) /2\},
	\end{split}
	\]
	for all $P_\theta\in\mathcal{F}(L,\boldsymbol{p},s_n)$ satisfying that $d(P_{\theta},P_{\theta_1})\leq d(P_{0},P_{\theta_1})/18$. 
	
	Let $K=N(\varepsilon_n^*/19,\mathcal{F}(L,\boldsymbol{p},s_n),d(\cdot,\cdot) )$ denote the covering number of set $\mathcal{F}(L,\boldsymbol{p},s_n)$, i.e., there exists $K$ Hellinger-balls with radius $\varepsilon_n^*/19$, that completely cover $\mathcal{F}(L,\boldsymbol{p},s_n)$.
	For any $\theta\in\mathcal{F}(L,\boldsymbol{p},s_n)$ (W.O.L.G, we assume $P_\theta$ belongs to the $k$th Hellinger ball centered at $P_{\theta_k}$), if $d(P_\theta, P_0)>\varepsilon_n^*$,
	then  we must have that  $d(P_0,P_{\theta_k})>(18/19)\varepsilon_n^*$ and there exists a testing function $\psi_k$, such that
	\[
	\begin{split}
	\mathbb{E}_{P_0}(\psi_k)&\leq \exp\{-nd^2(P_{\theta_k},P_{0}) /2\}\\
	&\leq \exp\{-(18^2/19^2/2)n\varepsilon_n^{*2}\},\\
	\mathbb{E}_{P_{\theta}}(1-\psi_k)&\leq \exp\{-nd^2(P_{\theta_k},P_{0}) /2\}\\
	&\leq \exp\{-n(d(P_0,P_{\theta})-\varepsilon_n^*/19)^2  /2\}\\
	&\leq \exp\{-(18^2/19^2/2)nd^2(P_0,P_{\theta})\}.
	\end{split}
	\]
	Now we define $\phi=\max_{k=1,\dots,K}\psi$.
	Thus we must have
	\[
	\begin{split}
	\mathbb{E}_{P_0}(\phi)&\leq\sum_k \mathbb{E}_{P_0}(\psi_k)\leq K \exp\{-(18^2/19^2/2)n\varepsilon_n^{*2}\}\\
	&\leq \exp\{-(18^2/19^2/2)n\varepsilon_n^{*2}-\log K\}.
	\end{split}
	\]
	Note that
	\begin{align}
	&\log K=\log N(\varepsilon_n^*/19,\mathcal{F}(L,\boldsymbol{p},s_n),d(\cdot,\cdot) )\nonumber\\
	&\leq \log N(\sqrt{8}\sigma_{\varepsilon}\varepsilon_n^*/19,\mathcal{F}(L,\boldsymbol{p},s_n),\|\cdot\|_\infty )\nonumber\\
	&\leq (s_n+1)\log(\frac{38}{\sqrt{8}\sigma_{\varepsilon}\varepsilon_n^*}(L+1)(N+1)^{2(L+1)})\nonumber\\
	&\leq C_0(s_n\log \frac{1}{\varepsilon_n^*} + s_n\log (L+1) + s_n(L+1)\log N)\nonumber\\
	& \leq s_n(L+1)\log n\log N\leq s^*(L+1)\log N\log^{2\delta} n \nonumber\\
	&\leq n\varepsilon_n^{*2}/4,\mbox{ for sufficiently large n,}\label{testrate}
	\end{align}
	where $C_0$ is some positive constant, the first inequality is due to the fact
	\begin{equation*}\begin{split}
	d^2(P_{\theta}, P_0) 
	\leq 1- \exp\{-\frac{1}{8\sigma^2_{\epsilon}}||f_0 - f_{\theta}||^2_\infty\}
	\end{split}
	\end{equation*}
	and $\varepsilon_n^*=o(1)$,
	the second inequality is due to Lemma 10 of \cite{Schmidt-Hieber2017Nonparametric}\footnote{Although \cite{Schmidt-Hieber2017Nonparametric} only focuses on ReLU network, its Lemma 10 could apply to any 1-Lipchitz continuous activation function.}, and the last inequality is due to $s_n\log(1/\varepsilon^*_n) \asymp s_n\log n$. 
	Therefore,
	\[
	\begin{split}
	\mathbb{E}_{P_0}(\phi)&\leq\sum_k P_0(\psi_k)\leq  \exp\{-C_1n\varepsilon_n^{*2}\},
	\end{split}
	\]
	for some $C_1=18^2/19^2/2-1/4$. On the other hand, for any $\theta$, such that $d(P_\theta,P_0)\geq \varepsilon_n^*$, say $P_\theta$ belongs to the $k$th Hellinger ball, then we have
	\[
	\begin{split}
	\mathbb{E}_{P_{\theta}}(1-\phi)&\leq \mathbb{E}_{P_{\theta}}(1-\psi_k)\leq\exp \{-C_2nd^2(P_0,P_{\theta})\},
	\end{split}
	\]
	where $C_2=18^2/19^2/2$. Hence we conclude the proof.
\end{proof}

Lemma \ref{lmdonsker} restates the  Donsker and Varadhan's representation for the $\mbox{KL}$ divergence, whose proof can be found in \cite{Boucheron2013Concentration}. 
\begin{lem}{\label{lmdonsker}}
	For any probability measure $\mu$ and any measurable function $h$ with $e^h \in L_1(\mu)$,
	$$
	\log \int e^{h(\eta)}\mu(d\eta) = \sup_{\rho}\left[\int h(\eta)\rho(d \eta) - \mbox{KL}(\rho||\mu)\right].
	$$
\end{lem}

{\noindent\bf Proof of Lemma 4.2}
\begin{proof}
	Denote $\Theta_n$ as the truncated parameter space $\{\theta: \sum^T_{i=1}1(\theta_i \neq 0) \leq s_n\}$, where $s_n$ is defined in Lemma \ref{lmtesting}. Noting that 
	\begin{equation} \label{eq:decomp}
	\int_{\theta \in \Theta} d^2(P_{\theta}, P_0)\widehat{q}(\theta)d\theta = \int_{\theta \in \Theta_n} d^2(P_{\theta}, P_0)\widehat{q}(\theta)d\theta + \int_{\theta \in \Theta_n^c} d^2(P_{\theta}, P_0)\widehat{q}(\theta)d\theta,
	\end{equation}
	it suffices to find upper bounds of the two components in RHS of ($\ref{eq:decomp}$).
	
	We start with the first component. Denote $\widetilde{\pi}(\theta)$ to be the truncated prior $\pi(\theta)$ on
	set $\Theta_n$, i.e., $\widetilde{\pi}(\theta)= \pi(\theta)1(\theta\in\Theta_n)/\pi(\Theta_n)$, then by Lemma \ref{lmtesting} and the same argument used in Theorem 3.1 of \cite{Pati2018on}, it could be shown
	\begin{equation}\label{eqn:song:1}
	\int_{\Theta_n}\eta(P_{\theta}, P_0)\widetilde{\pi}(\theta) d\theta \leq e^{C_0n\varepsilon^{*2}_n}, \mbox{w.h.p.}
	\end{equation}
	for some $C_0>0$, where $\log\eta(P_{\theta}, P_0) = l_n(P_{\theta}, P_0) + \frac{n}{3}d^2(P_{\theta}, P_0)$.
	We further denote the $\widehat{q}(\theta)$ restricted on $\Theta_n$ as $\widecheck{q}(\theta)$, i.e.,
	$\widecheck{q}(\theta)= \widehat q(\theta)1(\theta\in\Theta_n)/\widehat q(\Theta_n)$, then by Lemma \ref{lmdonsker} and (\ref{eqn:song:1}), w.h.p.,
	\begin{equation}\label{eqn:2}
	\begin{split}
	&\frac{n}{3\widehat q(\Theta_n)}\int_{\Theta_n} d^2(P_{\theta}, P_0)\widehat{q}(\theta)d\theta= 
	\frac{n}{3}\int_{\Theta_n} d^2(P_{\theta}, P_0)\widecheck{q}(\theta)d\theta \\
	\leq& Cn\varepsilon^{*2}_n + \mbox{KL}(\widecheck{q}(\theta)||\widetilde{\pi}(\theta)) - \int_{\Theta_n} l_n(P_{\theta}, P_0)\widecheck{q}(\theta) d\theta.\\
	\end{split}
	\end{equation}
	Furthermore,
	\begin{equation*}
	\begin{split}
	\mbox{KL}(\widecheck{q}(\theta)||\widetilde{\pi}(\theta))&=\frac{1}{\widehat q(\Theta_n)} \int_{\theta \in \Theta_n} \log \frac{\widehat{q}(\theta)}{\pi(\theta)} \widehat{q}(\theta)d\theta + \log \frac{\pi(\Theta_n)}{\widehat q(\Theta_n)}\\
	&=\frac{1}{\widehat q(\Theta_n)} \mbox{KL}(\widehat{q}(\theta)||\pi(\theta)) - \frac{1 }{\widehat q(\Theta_n)} \int_{\theta \in \Theta_n^c} \log \frac{\widehat{q}(\theta)}{\pi(\theta)} \widehat{q}(\theta)d\theta + \log \frac{\pi(\Theta_n)}{\widehat q(\Theta_n)},
	\end{split}
	\end{equation*}
	and similarly,
	\begin{equation*}
	\begin{split}
	\int_{\Theta_n} l_n(P_{\theta}, P_0)\widecheck{q}(\theta) d\theta  =\frac{1}{\widehat q(\Theta_n)} \int_{\Theta} l_n(P_{\theta}, P_0)\widehat{q}(\theta) d\theta-\frac{1}{\widehat q(\Theta_n)} \int_{\Theta_n^c} l_n(P_{\theta}, P_0)\widehat{q}(\theta) d\theta.
	\end{split}
	\end{equation*}
	Combining the above two equations together, we have
	\begin{equation}\label{eqn:1}
	\begin{split}
	&\frac{n}{3\widehat q(\Theta_n)}\int_{\Theta_n} d^2(P_{\theta}, P_0)\widehat{q}(\theta)d\theta\leq Cn\varepsilon_n^{*2}+ \mbox{KL}(\widecheck{q}(\theta)||\widetilde{\pi}(\theta))-\int_{\Theta_n} l_n(P_{\theta}, P_0)\widecheck{q}(\theta) d\theta\\
	=&Cn\varepsilon_n^{*2}+\frac{1}{\widehat q(\Theta_n)}\left( \mbox{KL}(\widehat{q}(\theta)||\pi(\theta))-\int_{\Theta} l_n(P_{\theta}, P_0)\widehat{q}(\theta) d\theta \right) \\
	&- \frac{1}{\widehat q(\Theta_n)}\left( \int_{ \Theta_n^c} \log \frac{\widehat{q}(\theta)}{\pi(\theta)} \widehat{q}(\theta)d\theta -\int_{\Theta_n^c} l_n(P_{\theta}, P_0)\widehat{q}(\theta) d\theta\right)+ \log \frac{\pi(\Theta_n)}{\widehat q(\Theta_n)}.
	\end{split}
	\end{equation}
	
	The second component of ($\ref{eq:decomp}$) trivially satisfies that $\int_{\theta \in \Theta_n^c} d^2(P_{\theta}, P_0)\widehat{q}(\theta)d\theta \leq \int_{\theta \in \Theta_n^c}\widehat{q}(\theta)d\theta =\widehat q(\Theta_n^c)$. Thus, together with (\ref{eqn:1}), we have that w.h.p.,
	\begin{equation}\label{eqn:sum}
	\begin{split}
	& \int d^2(P_{\theta}, P_0)\widehat{q}(\theta)d\theta 
	\leq 3\widehat q(\Theta_n)C\varepsilon^{*2}_n + \frac{3}{n}\left( \mbox{KL}(\widehat{q}(\theta)||\pi(\theta))-\int_{\Theta} l_n(P_{\theta}, P_0)\widehat{q}(\theta) d\theta \right) \\
	&+ \frac{3}{n} \int_{\Theta_n^c} l_n(P_{\theta}, P_0)\widehat{q}(\theta) d\theta+ \frac{3}{n} \int_{ \Theta_n^c} \log \frac{\pi(\theta)}{\widehat{q}(\theta)} \widehat{q}(\theta)d\theta + \frac{3\widehat q(\Theta_n)}{n}\log \frac{\pi(\Theta_n)}{\widehat q(\Theta_n)}+\widehat q(\Theta_n^c).
	\end{split}
	\end{equation}
	
	The second term in the RHS of (\ref{eqn:sum}) is bounded by $C'(r^*_n
	+ \xi^*_n)$ w.h.p., due to Lemma 4.1, where $C'$ is either positive constant or diverging sequence depending on whether $n(r_n^*+\xi_n^*)$ diverges.
	
	The third term in the RHS of (\ref{eqn:sum}) is bounded by 
	\[\begin{split}
	& \frac{3}{n}\int_{\Theta_n^c} l_n(P_{\theta}, P_0)\widehat{q}(\theta) d\theta\\
	= & \frac{3}{2n\sigma^2_{\epsilon}}\int_{\Theta_n^c} \left[\sum^n_{i=1}\epsilon_i^2 -\sum^n_{i=1} (\epsilon_i+f_0(X_i)-f_{\theta}(X_i))^2\right]\widehat{q}(\theta) d\theta\\
	= &  \frac{3}{2n\sigma^2_{\epsilon}}\int_{\Theta_n^c} \left[-2\sum^n_{i=1} (\epsilon_i\times(f_0(X_i)-f_{\theta}(X_i))-\sum^n_{i=1}(f_0(X_i)-f_{\theta}(X_i))^2\right]\widehat{q}(\theta) d\theta\\
	= &\frac{3}{2n\sigma^2_{\epsilon}}\left\{
	-2\sum^n_{i=1}\epsilon_i\int_{\Theta_n^c}(f_0(X_i)-f_{\theta}(X_i))\widehat{q}(\theta)d\theta - \int_{\Theta_n^c}\sum^n_{i=1}(f_0(X_i)-f_{\theta}(X_i))^2\widehat{q}(\theta) d\theta
	\right\}.
	\end{split}
	\]
	Conditional on $X_i$'s, $-2\sum^n_{i=1}\epsilon_i\int_{\Theta_n^c}(f_0(X_i)-f_{\theta}(X_i))\widehat{q}(\theta)d\theta$ follows a normal distribution $\mathcal{N}(0, V^2)$, where 
	$V^2=4\sigma^2_{\epsilon}\sum^n_{i=1}(\int_{\Theta_n^c}(f_0(X_i)-f_{\theta}(X_i))\widehat{q}(\theta)d\theta)^2\leq 4\sigma^2_{\epsilon} \int_{\Theta_n^c}\sum^n_{i=1}(f_0(X_i)-f_{\theta}(X_i))^2\widehat{q}(\theta)d\theta$. Thus conditional on $X_i$'s, the third term in the RHS of (\ref{eqn:sum}) is bounded by 
	\begin{equation} \label{eq:third}
	\frac{3}{2n\sigma^2_{\epsilon}}\left[\mathcal{N}(0,V^2)-\frac{V^2}{4\sigma^2_{\epsilon}}\right].
	\end{equation}
	Noting that $\mathcal{N}(0,V^2)=O_p(M_nV)$ for any diverging sequence $M_n$, (\ref{eq:third}) is further bounded, w.h.p., by
	\[
	\frac{3}{2n\sigma^2_{\epsilon}}(M_nV-\frac{V^2}{4\sigma^2_{\epsilon}})\leq \frac{3}{2n\sigma^2_{\epsilon}}\sigma^2_{\epsilon}M_n^2.
	\] 
	Therefore, the third term in the RHS of (\ref{eqn:sum}) can be bounded by $\varepsilon_n^{*2}$ w.h.p. (by choosing
	$M_n^2=n\varepsilon_n^{*2}$).

	The fourth term in the RHS of (\ref{eqn:sum}) is bounded by 
	\[
	\frac{3}{n} \int_{ \Theta_n^c} \log \frac{\pi(\theta)}{\widehat{q}(\theta)} \widehat{q}(\theta)d\theta\leq \frac{3}{n}\widehat q( \Theta_n^c)\log \frac{\pi( \Theta_n^c)}{\widehat q( \Theta_n^c)}
	\leq \frac{3}{n}\sup_{x\in(0,1)}[x\log(1/x)]=O(1/n).
	\]

	Similarly, the fifth term in the RHS of (\ref{eqn:sum}) is bounded by $O(1/n)$.
	
	For the last term in the RHS of (\ref{eqn:sum}), by Lemma \ref{5term} in below, w.h.p., $\widehat q(\Theta_n^c)\leq \varepsilon_n^{*2}$.
	
	Combine all the above result together, w.h.p.,
	\[
	\begin{split}
	& \int d^2(P_{\theta}, P_0)\widehat{q}(\theta)d\theta 
	\leq C\varepsilon^{*2}_n + \frac{3}{n}\left( \mbox{KL}(\widehat{q}(\theta)||\pi(\theta))-\int_{\Theta} l_n(P_{\theta}, P_0)\widehat{q}(\theta) d\theta \right) + O(1/n),
	\end{split}
	\]
	where $C$ is some constant.
\end{proof}

\begin{lem}[Chernoff bound for Poisson tail]\label{lemmad}
	Let $X\sim \mbox{poi}(\lambda)$ be a Poisson random variable. For any $x>\lambda$,
	\[ 
	P(X \geq x) \leq \frac{(e \lambda)^x e^{-\lambda}}{x^x}.
	\]
\end{lem}
\begin{lem}\label{5term}
	If $\lambda\leq T^{-1}\exp\{-M nr_n^*/s_n\}$ for any positive diverging sequence $M\rightarrow\infty$, then w.h.p., $\widehat q(\Theta_n^c)= O(\varepsilon_n^{*2})$.
\end{lem}

\begin{proof}
	By Lemma 4.1, we have that w.h.p.,
	\[\begin{split}
	&\mbox{KL}(\widehat q(\theta)||\pi(\theta|\lambda))+ \int_{\Theta} l_n(P_0, P_{\theta})\widehat q(\theta)d\theta =\inf_{q_{\theta} \in \mathcal Q}\Bigl\{ \mbox{KL}(q(\theta)||\pi(\theta|\lambda))
	+ \int_{\Theta} l_n(P_0, P_{\theta})q(\theta)(d\theta) \Bigr\}\\
	\leq& Cnr_n^* \quad(\mbox{Note that }r_n^*\asymp\xi^*_n)
	\end{split}\] where $C$ is either a constant or any diverging sequence, depending on whether $n r_n^*$ diverges.
	By the similar argument used in the proof of Lemma 4.1, 
	\[
	\int_{\Theta} l_n(P_0, P_{\theta})\widehat q(\theta)d\theta \leq \frac{1}{2\sigma_\epsilon^2}\left(\int_{\Theta} ||f_{\theta}(X) - f_0(X)||^2_2 \widehat q(\theta)(d\theta)
	+Z
	\right) 
	\]
	where $Z$ is a normal distributed $\mathcal{N}(0, \sigma^2_{\epsilon}c_0')$, where $c_0'\leq c_0=\int_{\Theta} ||f_{\theta}(X) - f_0(X)||^2_2 \widehat q(\theta)(d\theta)$. Therefore, $-\int_{\Theta} l_n(P_0, P_{\theta})\widehat q(\theta)d\theta = (1/2\sigma_\epsilon^2)[-c_0+O_p(\sqrt{c_0})]$, and 
	$\mbox{KL}(\widehat q(\theta)||\pi(\theta|\lambda))\leq Cnr_n^*+(1/2\sigma_\epsilon^2)[-c_0+O_p(\sqrt{c_0})]$. Since $Cnr_n^*\rightarrow\infty$, we must have w.h.p., $\mbox{KL}(\widehat q(\theta)||\pi(\theta|\lambda))\leq Cnr_n^*/2$.
	On the other hand, 
	\begin{equation}\label{up1}
	\begin{split}
	&\mbox{KL}(\widehat q(\theta)||\pi(\theta|\lambda))=\sum_{i=1}^T\mbox{KL}(\widehat q(\theta_i)||\pi(\theta_i|\lambda))
	\geq \sum_{i=1}^T\mbox{KL}(\widehat q(\gamma_i)||\pi(\gamma_i|\lambda))\\
	=&\sum_{i=1}^T \left[\widehat q(\gamma_i=1)\log \frac{\widehat q(\gamma_i=1)}{\lambda}+ \widehat q(\gamma_i=0)\log \frac{\widehat q(\gamma_i=0)}{1-\lambda}\right].
	\end{split}
	\end{equation}
	Let us choose $\lambda_0=1/T$, and $A=\{i: \widehat q(\gamma_i=1)\geq \lambda_0\}$, then the above inequality (\ref{up1}) implies that $\sum_{i\in A}\widehat q(\gamma_i=1)\log (\lambda_0/\lambda)\leq Cnr_n^*/2$. Noting that $\lambda\leq T^{-1}\exp\{-M nr_n^*/s_n\}$, it further implies $\sum_{i\in A}\widehat q(\gamma_i=1)\leq s_n/M\prec s_n$.
	
	Under distribution $\widehat q$, by Bernstein inequality,  
	\[
	\begin{split}
	& Pr(\sum_{i\in A}\gamma_i\geq 2s_n/3)\leq
	Pr(\sum_{i\in A}\gamma_i\geq s_n/2+\sum_{i\in A}\mathbb{E}(\gamma_i))\leq 
	\exp \left( -\frac{s_n^2/8}{\sum_{i\in A}\mathbb{E} [\gamma_i^2]+s_n/6} \right)\\=&\exp \left( -\frac{s_n^2/8}{\sum_{i\in A}\widehat q(\gamma_i=1)+s_n/6} \right)
	\leq \exp \left( -cs_n\right)=O(\varepsilon_n^{*2})
	\end{split}\]
	for some constant $c>0$, where the last inequality holds since $\log(1/\varepsilon_n^{*2}) = O(\log n)\prec s_n$.
	
	Under distribution $\widehat q$, $\sum_{i\notin A}\gamma_i$ is stochastically smaller than $Bin(T,\lambda_0)$. 
	Since $T\rightarrow\infty$, then by Lemma \ref{lemmad},  
	\[
	\begin{split}
	&Pr(\sum_{i\notin A}\gamma_i\geq s_n/3)\leq Pr(Bin(T,\lambda_0)\geq s_n/3)\rightarrow Pr(\mbox{poi}(1)\geq s_n/3)\\
	= & O(\exp\{-c's_n\})=O(\varepsilon_n^{*2})
	\end{split}
	\]
	for some $c'>0$. Trivially, it implies that w.h.p, $Pr(\sum_{i}\gamma_i\geq s_n)=O(\varepsilon_n^{*2})$ for VB posterior $\widehat q$.
	
\end{proof}

\subsection{Main theorem}

\begin{thm}
	Under Conditions 4.1-4.2, 4.4 and set $-\log\lambda =\log(T)+\delta[(L+1)\log N + \log \sqrt{n}p]$ for any constant $\delta>0$, we then have that w.h.p.,
	\[\int_{\Theta} d^2(P_{\theta}, P_0)\widehat{q}(\theta)d\theta \leq C\varepsilon^{*2}_n + C'(r_n^* +\xi^*_n),\] where $C$ is some positive constant and $C'$ is any diverging sequence. If $\|f_0\|_\infty<F$, and we truncated the VB posterior on $\Theta_F=\{\theta: \|f_\theta\|_\infty\leq F\}$, i.e., $\widehat q_F\propto\widehat q1(\theta\in\Theta_F)$, then, w.h.p.,
	\[
	\int_{\Theta_F} \mathbb{E}_X|f_{\theta}(X) - f_0(X)|^2\widehat{q}_F(\theta)d\theta\leq\frac{C\varepsilon^{*2}_n + C'(r_n^* +\xi^*_n)}{C_F \widehat q(\Theta_F)}
	\]
	where $C_F = [1-\exp(-4F^2/8\sigma^2_{\epsilon})]/4F^2$, and $\widehat q(\Theta_F)$ is the VB posterior mass of $\Theta_F$.
\end{thm}
\begin{proof}
	The convergence under squared Hellinger distance is directly result of Lemma 4.1 and 4.2, by simply checking the choice of $\lambda$ satisfies required conditions. The convergence under $L_2$ distance relies on inequality $d^2(P_{\theta}, P_0)\geq C_F\mathbb{E}_X|f_{\theta}(X) - f_0(X)|^2$ for $C_F = [1-\exp(-4F^2/8\sigma^2_{\epsilon})]/4F^2$ when both $f_\theta$ and $f_0$ are bounded by $F$. Then, w.h.p,
	\[
	\begin{split}
	&\int_{\Theta_F}  \mathbb{E}_X|f_{\theta}(X) - f_0(X)|^2\widehat{q}_F(\theta)d\theta\leq C_F^{-1} \int_{\Theta_F}  d^2(P_\theta, P_0)\widehat{q}_F(\theta)d\theta\\
	\leq &  \frac{1}{C_F \widehat q(\Theta_F)}\int_{\Theta} d^2(P_\theta, P_0)\widehat{q}(\theta)d\theta
	\leq \frac{C\varepsilon^{*2}_n + C'(r_n^* +\xi^*_n)}{C_F \widehat q(\Theta_F)}.
	\end{split}
	\]
	

\end{proof}

\section{Additional experimental results}
\subsection{Comparison between Bernoulli variable and the Gumbel softmax approximation}
Denote $\gamma_i \sim \mbox{Bern}(\phi_i)$ and $\widetilde{\gamma}_i \sim \mbox{Gumbel-softmax}(\phi_i, \tau)$, then we have that
\[\begin{split}
&\widetilde{\gamma}_i:=g_\tau(\phi_i,u_i) = (1+\exp(-\eta_i/\tau))^{-1}, \quad\mbox{where } \eta_i=\log \frac{\phi_i}{1-\phi_i} + \log \frac{u_i}{1-u_i}, \quad u_i \sim \mathcal{U}(0,1),\\
&\gamma_i := g(\phi_i,u_i) = 1(u_i\leq \phi_i) \quad\mbox{where }  u_i \sim \mathcal{U}(0,1).
\end{split}
\]
Fig \ref{fig:gumbel} demonstrates the functional convergence of  $g_\tau$ towards $g$ as $\tau$ goes to zero. In Fig \ref{fig:gumbel}(a),  by fixing $\phi_i(=0.9)$, we show $g_{\tau}$ converges to $g$ as a function of $u_i$. Fig \ref{fig:gumbel} (b) demonstrates that $g_{\tau}$ converges to $g$ as a function of $\alpha_i = \log(\phi_i/(1-\phi_i))$ when $u_i(=0.2)$ is fixed. These two figures show that as $\tau\rightarrow 0$, $g_\tau\rightarrow g$.
Formally, \citet{Maddison17Concrete} rigorously proved that 
$\widetilde{\gamma}_i$ converges to $\gamma_i$ in distribution as $\tau$ approaches 0.

\begin{figure*}[!htb] 
	\centering
	\begin{subfigure}[b]{.4\textwidth}
		\centering
		\includegraphics[width=1\linewidth]{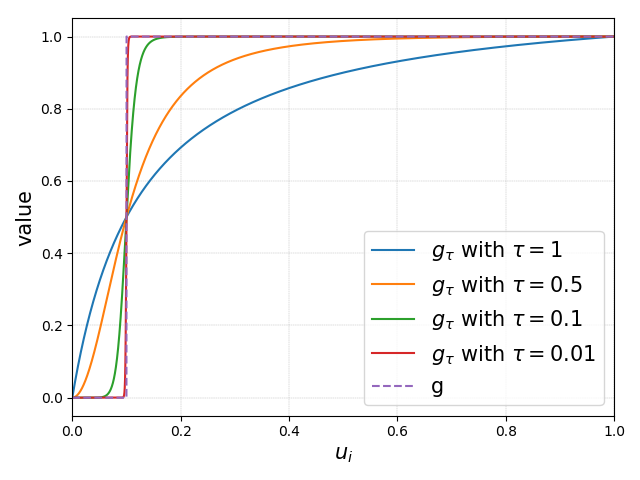}
		\caption{Fix $\phi_i=0.9$.}
	\end{subfigure}
	\begin{subfigure}[b]{.4\textwidth}
		\centering
		\includegraphics[width=1\linewidth]{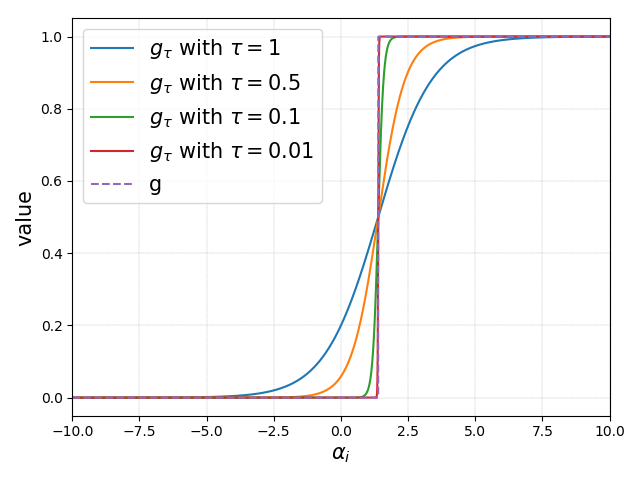}
		\caption{Fix $u_i=0.2$.}
	\end{subfigure}
	\caption{The convergence of $g_{\tau}$ towards $g$ as $\tau$ approaches 0.}
	\label{fig:gumbel}
\end{figure*}

\subsection{Algorithm implementation details for the numerical experiments}
\paragraph{Initialization}
As mentioned by \cite{sonderby2016How} and \cite{Molchanov2017Variational},  training sparse BNN with random initialization may lead to bad performance, since many of the weights could be pruned too early. In our case, we assign each of the weights and biases a inclusion variable, which could reduce to zero quickly in the early optimization stage if we randomly initialize them. As a consequence, we deliberately initialize $\phi_i$ to be close to 1 in our experiments. This initialization strategy ensures the training starts from a fully connected neural network, which is similar to start training from a pre-trained fully connected network as mentioned in \cite{Molchanov2017Variational}. The other two parameters $\mu_i$ and $\sigma_i$ are initialized randomly.

\paragraph{Other implementation details in simulation studies}
We set $K=1$ and $\mbox{learning rate} = 5\times 10^{-3}$ during training. For Simulation I, we choose batch size $m = 1024$ and $m = 128$ for (A) and (B) respectively, and run 10000 epochs for both cases. For simulation II, we use $m=512$ and run 7000 epochs. 
Although it is common to set up an annealing schedule for temperature parameter $\tau$, we don't observe any significant performance improvement compared to setting $\tau$ as a constant, therefore we choose $\tau=0.5$ in all of our experiments. The optimization method used is Adam. 

The implementation details for UCI datasets and MNIST can be found in Section \ref{sec:uci} and \ref{sec:mnist} respectively.

\subsection{Toy example: linear regression}
In this section, we aim to demonstrate that there is little difference between the results using inverse-CDF reparameterization and Gumbel-softmax approximation via a toy example.


Consider a linear regression model:
\[
Y_i =X^T_i \beta + \epsilon_i, \quad \epsilon_i \sim \mathcal{N}(0, 1), \quad i=1, \ldots, n,
\]

We simulate a dataset with 1000 observations and 200 predictors, where $\beta_{50}=\beta_{100}=\beta_{150} = 10$, $\beta_{75}= \beta_{125}= -10$ and $\beta_j=0$ for all other $j$.

A spike-and-slab prior is imposed on $\beta$ such that
\[
\beta_j|\gamma_j \sim \gamma_j\mathcal{N}(0, \sigma^2_0) + (1-\gamma_j)\delta_0, \quad \gamma_j \sim \mbox{Bern}(\lambda), 
\]
for $j = 1,\ldots, 200$, where $\sigma_0=5$ and $\lambda=0.03$. The variational distribution $q(\beta)\mathcal Q$ is chosen as
\[
\beta_j|\gamma_j \sim \gamma_j\mathcal{N}(\mu_j, \sigma^2_j) + (1-\gamma_j)\delta_0, \quad \gamma_j \sim \mbox{Bern}(\phi_j).
\]

We use both Gumbel-softmax approximation and inverse-CDF reparameterization for the stochastic optimization of ELBO, and plot posterior mean $\mathbb{E}_{\widehat{q}(\beta)}(\beta_j|\gamma_j)$ (blue curve) against the true value (red curve). Figure $\ref{fig:linear}$ shows that inverse-CDF reparameterization exhibits only slightly higher error in estimating zero coefficients than the Gumbel-softmax approximation, which indicates the two methods has little difference on this toy example.

\begin{figure*}[!htb]
	\centering
	\begin{subfigure}[b]{.45\textwidth}
		\centering
		\includegraphics[width=1\linewidth]{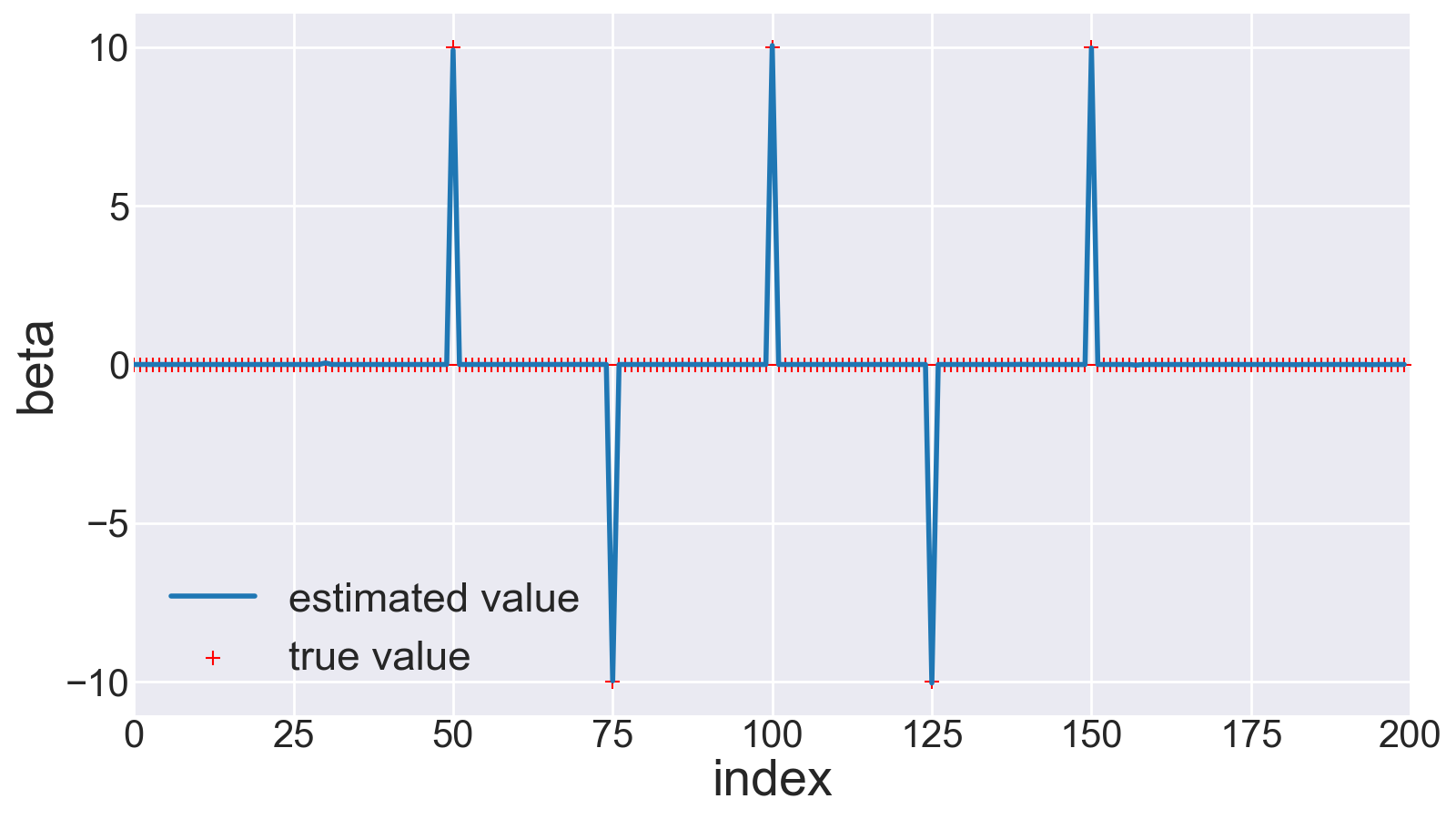}
		\caption{ Gumbel-softmax reparametrization }
	\end{subfigure}%
	\begin{subfigure}[b]{0.45\textwidth}
		\centering
		\includegraphics[width=1\linewidth]{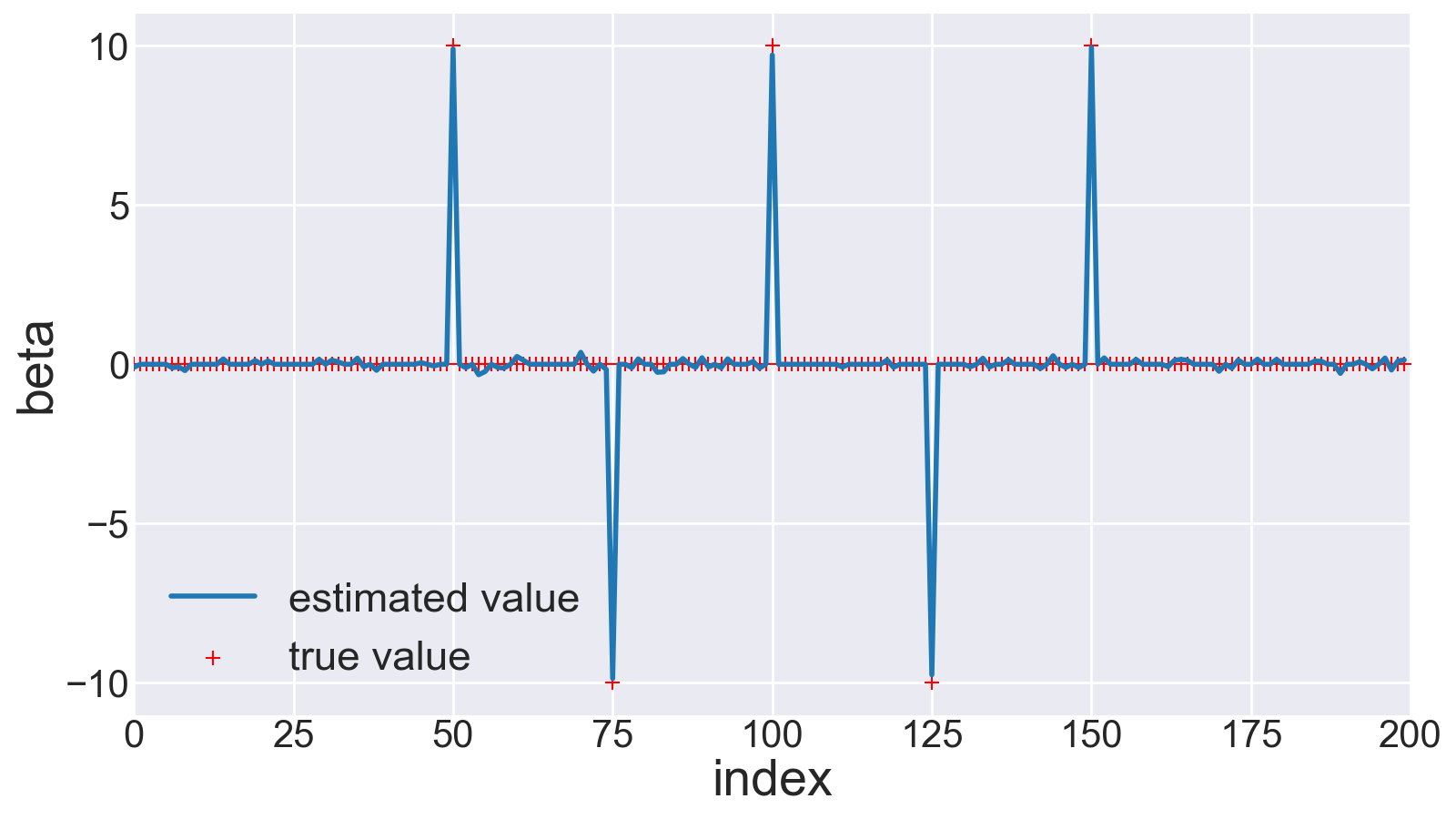}
		\caption{ Inverse-CDF reparametrization }
	\end{subfigure}
	\caption{Linear regression}
	\label{fig:linear}
\end{figure*}

\subsection{Teacher student networks}
The network parameter $\theta$ for the sparse teacher network setting (B) is set as following:
$W = \{W_{1,11}=W_{1,12}=W_{2,11}=W_{2,12}=2.5, W_{1,21}=W_{1,22}=W_{2,21}=W_{2,22}=1.5, W_{3,11}=3 \mbox{ and }W_{3,21}=2\}$; $b=\{b_{1,1}=b_{2,1}=b_{3,1}=1 \mbox{ and } b_{1,2}=b_{2,2}=-1\}$.

Figure \ref{fig:denseteacher_lambda} displays the simulation result for simulation I under dense teacher network (A) setting. Unlike the result under sparse teacher network (B), the testing accuracy seems monotonically increases as $\lambda$ increases (i.e., posterior network gets denser). However, as shown, the increasing of testing performance is rather slow, which indicates that introducing sparsity has few negative impact to the testing accuracy.

\begin{figure*}[!htb] 
	\centering
	\begin{subfigure}[b]{.33\textwidth}
		\centering
		\includegraphics[width=1\linewidth]{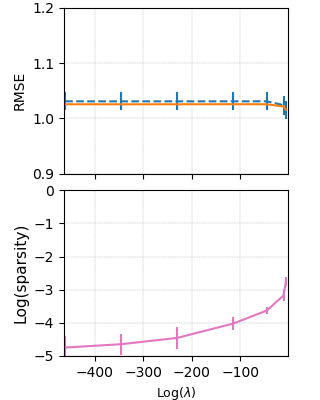}
		\caption{$\lambda \leq \lambda_{opt}$.}
	\end{subfigure}
	\begin{subfigure}[b]{.33\textwidth}
		\centering
		\includegraphics[width=1\linewidth]{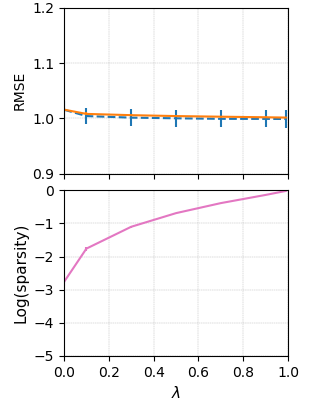}
		\caption{$\lambda \geq \lambda_{opt}$.}
	\end{subfigure}
	\begin{subfigure}[b]{.32\textwidth}
		\centering
		\includegraphics[width=1\linewidth]{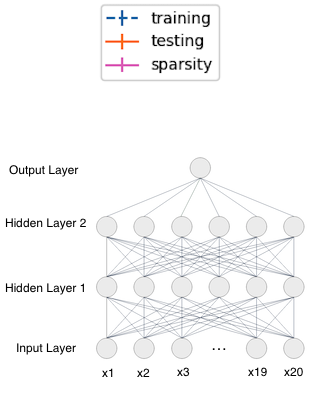}
		\caption{Dense teacher network.}
	\end{subfigure}%
	\caption{(a) $\lambda = \{10^{-200}, 10^{-150}, 10^{-100}, 10^{-50}, 10^{-20}, 10^{-5}, \lambda_{opt}\}$. (b) $\lambda = \{\lambda_{opt}, 0.1, 0.3, 0.5,$ $ 0.7, 0.9, 0.99\}$. (c) The structure of the target dense teacher network.}
	\label{fig:denseteacher_lambda}
\end{figure*}

\paragraph{Coverage rate} In this paragraph, we explain the details of how we compute the coverage rate values of Bayesian intervals reported in the main text. A fixed point $(x_1^{(*)},\dots, x_p^{(*)})'$ is prespecified, and let $x^{(1)},\dots,x^{(1000)}$ be 1000 equidistant points from $-1$ to $1$. In each run, we compute the Bayesian credible intervals of response means (estimated by 600 Monte Carlo samples) for 1000 different input $x$'s: $(x^{(1)},x_2^{(*)},\dots, x_p^{(*)}), \dots, (x^{(1000)},x_2^{(*)},\dots, x_p^{(*)})$. It is repeated by 60 times and the average coverage rate (over all different $x$'s and 60 runs) is reported. Similarly, we replace $x_2^{(*)}$ (or $x_3^{(*)}$) by $x^{(i)}$ ($i=1,\dots, 1000$), and compute their average coverage rate. 
The complete coverage rate results are shown in Table \ref{tb:cover}. Note that Table 1 in the main text shows $95 \%$ coverage of $x_3$ for (A) and $95 \%$ coverage of $x_1$ for (B).

\begin{table}[!htb]
	\caption{Coverage rates for teacher networks.}
	\label{tb:cover}
	\small
	\centering
	\begin{tabular}{llcccccc}
		\toprule
		&  & \multicolumn{3}{c}{\textbf{90 \% coverage (\%)}} & \multicolumn{3}{c}{\textbf{95\% coverage (\%)}} \\ \cmidrule(lr){3-5} \cmidrule(lr){6-8} 
		&\textbf{Method}  & $\boldsymbol{x_1}$ & $\boldsymbol{x_2}$ & $\boldsymbol{x_3}$ & $\boldsymbol{x_1}$ & $\boldsymbol{x_2}$ & $\boldsymbol{x_3}$ \\ 
		\midrule
		\multirow{4}{*}{\rotatebox[origin=c]{90}{Dense}} & SVBNN & 93.8 $\pm$ 2.84 & 93.1 $\pm$ 4.93 & 93.1 $\pm$ 2.96 & 97.9 $\pm$ 1.01 & 97.9 $\pm$ 1.69 & 97.5 $\pm$ 1.71 \\
		& VBNN & 85.8 $\pm$ 2.51 & 82.4 $\pm$ 2.62 & 86.3 $\pm$ 1.88 & 92.7 $\pm$ 2.83 & 91.3 $\pm$ 2.61 & 91.4 $\pm$ 2.43 \\
		& VD & 61.3 $\pm$ 2.40 & 60.0 $\pm$ 2.79 & 64.9 $\pm$ 6.17 & 74.9 $\pm$ 1.79 & 71.8 $\pm$ 2.33 & 76.4 $\pm$ 4.75 \\ 
		& HS-BNN &83.1 $\pm$ 1.67 &80.0 $\pm$ 1.21 &76.9 $\pm$ 1.70&88.1 $\pm$ 1.13 &84.1 $\pm$ 1.48 & 83.5 $\pm$ 0.78\\\midrule
		\multirow{4}{*}{\rotatebox[origin=c]{90}{Sparse}} & SVBNN & 92.3 $\pm$ 8.61 & 94.6 $\pm$ 5.37 & 98.3 $\pm$ 0.00 & 96.4 $\pm$ 4.73 & 97.7 $\pm$ 3.71 & 100 $\pm$ 0.00 \\
		& VBNN & 86.7 $\pm$ 10.9 & 87.0 $\pm$ 11.3 & 93.3 $\pm$ 0.00 & 90.7 $\pm$ 8.15 & 91.9 $\pm$ 9.21 & 96.7 $\pm$ 0.00 \\
		& VD & 65.2 $\pm$ 0.08 & 63.7 $\pm$ 6.58 & 65.9 $\pm$ 0.83 & 75.5 $\pm$ 7.81 & 74.6 $\pm$ 7.79 & 76.6 $\pm$ 0.40 \\
		&HS-BNN &59.0 $\pm$ 8.52 &59.4 $\pm$ 4.38 & 56.6 $\pm$ 2.06&67.0 $\pm$ 8.54 & 68.2 $\pm$ 3.62& 66.5 $\pm$ 1.86\\
		\bottomrule
	\end{tabular}
\end{table}

\subsection{Real data regression experiment: UCI datasets} \label{sec:uci}
We follow the experimental protocols of \cite{Lobato2015Probabilistic}, and choose five datasets for the experiment. For the small datasets "Kin8nm", "Naval", "Power Plant" and "wine", we choose a single-hidden-layer ReLU network with 50 hidden units. We randomly select $90\%$ and $10\%$ for training and testing respectively, and this random split process is repeated for 20 times (to obtain standard deviations for our results). We choose minibatch size $m=128$, $\mbox{learning rate}=10^{-3}$ and run 500 epochs for "Naval", "Power Plant" and "Wine", 800 epochs for "Kin8nm". For the large dataset "Year", we use a single-hidden-layer ReLU network with 100 hidden units, and the evaluation is conducted on a single split. We choose $m=256$, $\mbox{learning rate}=10^{-3}$ and run 100 epochs. For all the five datasets, $\lambda$ is chosen as $\lambda_{opt}$: $\log(\lambda^{-1}_{opt})=\log(T)+0.1[(L+1)\log N + \log \sqrt{n}p]$, which is the same as other numerical studies. We let $\sigma^2_0=2$ and use grid search to find $\sigma_{\epsilon}$ that yields the best prediction accuracy. Adam is used for all the datasets in the experiment.

We report the testing squared root MSE (RMSE) based on $\widehat{f}_H$ (defined in the main text) with $H=30$, and also report the posterior network sparsity $\widehat{s}=\sum^T_{i=1}\phi_i/T$. For the purpose of comparison, we list the results by Horseshoe BNN (HS-BNN) \citep{ghosh2017Model} and probalistic backpropagation (PBP) \citep{Lobato2015Probabilistic}. Table \ref{tb:uci} demonstrates that our method achieves best prediction accuracy for all the datasets with a sparse structure.

\begin{table}[!htb]
	\caption{Results on UCI regression datasets.}
	\label{tb:uci}
	\small
	\centering
	\begin{tabular}{llcccc}
		\toprule
		&  & \multicolumn{3}{c}{Test RMSE} & Posterior sparsity(\%) \\ \cmidrule(lr){3-5} \cmidrule(lr){6-6} 
		Dataset & $n (p)$ & SVBNN & HS-BNN & PBP & SVBNN \\
		\midrule
		Kin8nm & 8192 (8) & 0.08 $\pm$ 0.00 & 0.08 $\pm$ 0.00 & 0.10 $\pm$ 0.00 & 64.5 $\pm$ 1.85  \\
		Naval & 11934 (16) & 0.00 $\pm$ 0.00 & 0.00 $\pm$ 0.00 & 0.01 $\pm$ 0.00 & 82.9 $\pm$ 1.31 \\
		Power Plant & 9568 (4) & 4.01 $\pm$ 0.18 & 4.03 $\pm$ 0.15 & 4.12 $\pm$ 0.03 & 56.6 $\pm$ 3.13 \\
		Wine & 1599 (11) & 0.62 $\pm$ 0.04 & 0.63 $\pm$ 0.04 & 0.64 $\pm$ 0.01 & 59.9 $\pm$ 4.92  \\
		Year & 515345 (90) &8.87 $\pm$ NA  & 9.26 $\pm$ NA & 8.88 $\pm$ NA & 20.8 $\pm$ NA   \\
		\bottomrule
	\end{tabular}
\end{table}

\subsection{Real data classification experiment: MNIST dataset} \label{sec:mnist}

The MNIST data is normalized by mean equaling 0.1306 and standard deviation equaling 0.3081. For all methods, we choose the same minibatch size $m = 256$, $\mbox{learning rate} = 5\times 10^{-3}$ for our method and $3 \times 10^{-3}$ for the others, total number of epochs is 400 and the optimization algorithm is RMSprop. AGP is pre-specified at $5\%$ sparsity level.

\begin{figure*}[!htb] 
	\centering
	\begin{subfigure}[b]{.45\textwidth}
		\centering
		\includegraphics[width=1\linewidth]{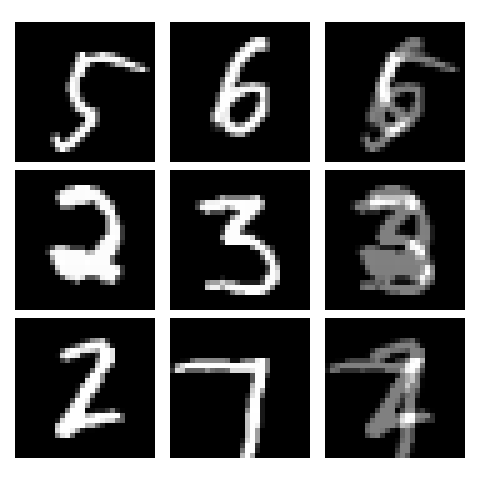}
		\caption{Overlaid images (on the last column)}
	\end{subfigure}
	\begin{subfigure}[b]{.45\textwidth}
		\centering
		\includegraphics[width=1\linewidth]{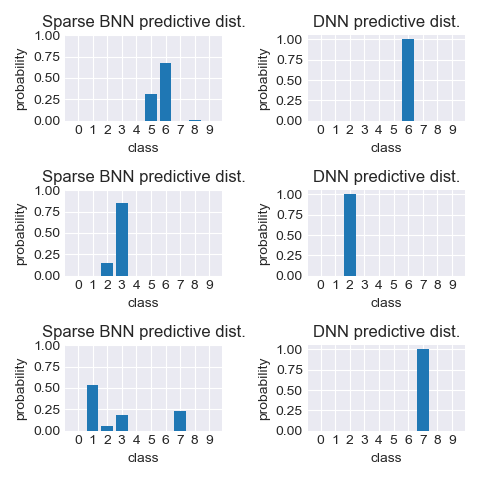}
		\caption{Predictive distribution for overlaid images}
	\end{subfigure}
	\caption{Top row of (b) exhibits the predictive distribution for the top overlaid image, which is made by 5 and 6; Middle row of (b) exhibits the predictive distribution for the middle overlaid image, which is made by 2 and 3; Bottom row of (b) exhibits the predictive distribution for the bottom overlaid image, which is made by 2 and 7.}
	\label{fig:mnist_uq}
\end{figure*}

Besides the testing accuracy reported in the main text, we also examine our method's ability of uncertainty quantification for MNIST classification task. We first create ambiguous images by overlaying two examples from the testing set as shown in Figure \ref{fig:mnist_uq} (a). To perform uncertainty quantification using our method, for each of the overlaid images, we generate $\theta_h$ from the VB posterior $\widehat{q}(\theta)$ for $h=1, \ldots, 100$, and calculate the associated predictive probability vector $f_{\theta_h}(x) \in \mathbb{R}^{10}$ where $x$ is the overlaid image input, and then use the estimated posterior mean $\widehat{f}(x)=\sum^{100}_{h=1}f_{\theta_h}(x)/100$ as the Bayesian predictive probability vector. As a comparison, we also calculate the predictive probability vector for each overlaid image using AGP as a frequentist benchmark. Figure \ref{fig:mnist_uq} (b) shows frequentist method gives almost a deterministic answer (i.e., predictive probability is almost 1 for certain digit) that is obviously unsatisfactory for this task, while our VB method is capable of providing knowledge of certainty on these out-of-domain inputs, which demonstrates the advantage of Bayesian method in uncertainty quantification on the classification task.

\subsection{Illustration of CNN: Fashion-MNIST dataset} \label{sec:fmnist}
In this section, we perform an experiment on a more complex task, the Fashion-MNIST dataset. To illustrate the usage of our method beyond feedforward networks, we consider using a 2-Conv-2-FC network: The feature maps for the convolutional layers are set to be 32 and 64, and the filter size are $5 \times 5$ and $3 \times 3$ respectively. The paddings are 2 for both layers and the it has a $2 \times 2$ max pooling for each of the layers; The fully-connected layers have $64 \times 8 \times 8$ neurons. The activation functions are all ReLUs. The dataset is prepocessed by random horizontal flip. The batchsize is 1024, learning rate is 0.001, and Adam is used for optimization. We run the experiment for 150 epochs.

We use both SVBNN and VBNN for this task. In particular, the VBNN, which uses normal prior and variational distributions, is the full Bayesian method without compressing, and can be regarded as the baseline for our method. Figure \ref{fig:fmnist} exhibits our method attains higher accuracy as epoch increases and then decreases as the sparisty goes down. Meanwhile, the baseline method - full BNN suffers from overfitting after 80 epochs.

\begin{figure*}[bt] 
	\centering   
	\begin{subfigure}[b]{.4\textwidth}
		\centering
		\includegraphics[width=1\linewidth]{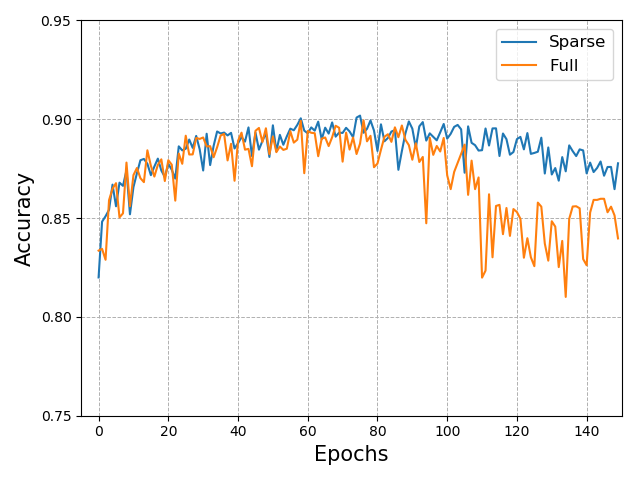}
		\caption{Accuracy.}
	\end{subfigure}
	\begin{subfigure}[b]{.4\textwidth}
		\centering
		\includegraphics[width=1\linewidth]{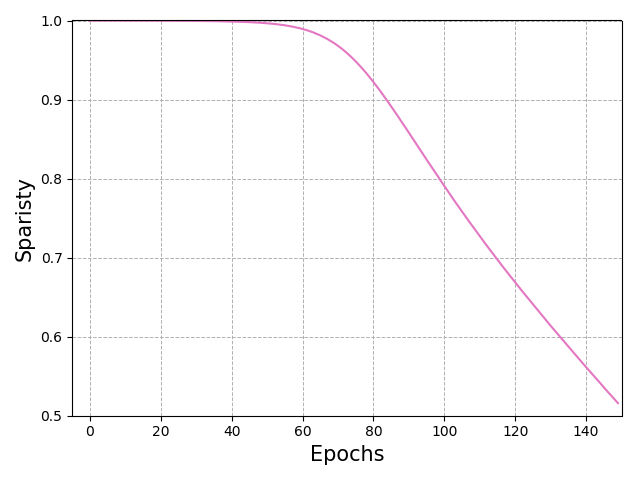}
		\caption{Sparsity.}
	\end{subfigure}
	\caption{Fashion-MNIST experiment.}
	\label{fig:fmnist}
\end{figure*}

\end{document}